%% file: main.tex
\documentclass{article}

\usepackage[final]{neurips_2022}

\usepackage[utf8]{inputenc} 
\usepackage[T1]{fontenc}    
\usepackage[breaklinks,colorlinks,
            linkcolor=red,
            anchorcolor=blue,
            citecolor=blue
            ]{hyperref}
\usepackage{hyperref}       
\usepackage{url}            
\usepackage{booktabs}       
\usepackage{amsfonts}       
\usepackage{nicefrac}       
\usepackage{microtype}      
\usepackage[dvipsnames]{xcolor}         
\usepackage{bm}
\usepackage{amsmath,amssymb}
\usepackage{amsthm}
\usepackage{enumerate}
\usepackage{graphicx}
\usepackage{diagbox}
\usepackage{multirow}
\usepackage{verbatim}
\newtheorem{theorem}{Theorem}[section]
\newtheorem{lemma}[theorem]{Lemma}

\newtheorem{definition}[theorem]{Definition}
\newtheorem{remark}[theorem]{Remark}
\newtheorem{proposition}[theorem]{Proposition}

\def\##1\#{\begin{align}#1\end{align}}
\def\$#1\${\begin{align*}#1\end{align*}}
\newcommand{\R}{\mathbb{R}}
\newcommand{\B}{\mathcal{B}}
\renewcommand{\L}{\mathcal{L}}
\renewcommand{\O}{\mathcal{O}}
\newcommand{\Olog}{\widetilde{\mathcal{O}}}
\newcommand{\eps}{\varepsilon}
\newcommand{\floor}[1]{\lfloor#1 \rfloor}

\title{Why Robust Generalization in Deep Learning is Difficult: Perspective of Expressive Power}

\newcommand*\samethanks[1][\value{footnote}]{\footnotemark[#1]}
\author{%
    \textbf{Binghui Li}$^{1,6,7,}$\thanks{Equal contribution.}\quad \textbf{Jikai Jin}$^{2,}$\samethanks\quad \textbf{Han Zhong}$^{3}$\quad
    \textbf{John E. Hopcroft}$^{4}$\quad
    \textbf{Liwei Wang}$^{3,5,}\thanks{Corresponding author.}$\\
    $^1$School of EECS, Peking University\\ $^2$School of Mathematical Sciences, Peking University\\ $^3$Center for Data Science, Peking University\quad$^4$Cornell University\\
    $^5$Key Laboratory of Machine Perception, MOE, School of Intelligence Science and Technology,\\Peking University\quad
    $^6$Peng Cheng Laboratory\quad$^7$Pazhou Laboratory (Huangpu)\\
    \texttt{\footnotesize \{libinghui,jkjin\}@pku.edu.cn,\quad hanzhong@stu.pku.edu.cn,} \\
    \texttt{\footnotesize jeh17@cornell.edu,\quad wanglw@cis.pku.edu.cn}\\
}

\begin{document}

\maketitle

\begin{abstract}

    It is well-known that modern neural networks are vulnerable to adversarial examples. To mitigate this problem, a series of robust learning algorithms have been proposed. However, although the robust training error can be near zero via some methods, all existing algorithms lead to a high robust generalization error. In this paper, we provide a theoretical understanding of this puzzling phenomenon from the perspective of expressive power for deep neural networks. Specifically, for binary classification problems with well-separated data, we show that, for ReLU networks, while mild over-parameterization is sufficient for high robust training accuracy, there exists a constant robust generalization gap unless the size of the neural network is exponential in the data dimension $d$. This result holds even if the data is linear separable (which means achieving standard generalization is easy), and more generally for any parameterized function classes as long as their VC dimension is at most polynomial in the number of parameters. Moreover, we establish an improved upper bound of $\exp({\mathcal{O}}(k))$ for the network size to achieve low robust generalization error when the data lies on a manifold with intrinsic dimension $k$ ($k \ll d$). Nonetheless, we also have a lower bound that grows exponentially with respect to $k$ --- the curse of dimensionality is inevitable. By demonstrating an exponential separation between the network size for achieving low robust training and generalization error, our results reveal that the hardness of robust generalization may stem from the expressive power of practical models.
    
\end{abstract}

\input{preprint_intro}

\input{preprint_train}

\input{preprint_general}

\input{preprint_linear}

\input{preprint_manifold}

\input{conclusion}

\section*{Acknowledgement}
We thank all the anonymous reviewers for their constructive comments. This work is supported by National Science Foundation of China (NSFC62276005), The Major Key Project of PCL (PCL2021A12), Exploratory Research Project of Zhejiang Lab (No. 2022RC0AN02), and Project 2020BD006 supported by PKUBaidu Fund. Binghui Li is partially supported by National Innovation Training Program of China. Jikai Jin is partially supported by the elite undergraduate training program of School of Mathematical Sciences in Peking University.

\bibliographystyle{ims}
\bibliography{graphbib}

\newpage 
\appendix

\input{appendix}

\end{document}

%% file: preprint_intro.tex
\section{Introduction}
Deep neural networks have achieved remarkable success in a variety of disciplines including computer vision \citep{voulodimos2018deep}, natural language processing \citep{devlin2018bert} as well as scientific and engineering applications \citep{jumper2021highly}. However, it is observed that neural networks are often sensitive to small adversarial attacks \citep{biggio2013evasion,szegedy2013intriguing,goodfellow2014explaining}, which potentially gives rise to reliability and security problems in real-world applications.

In light of this pitfall, it is 
highly desirable to obtain classifiers that are robust to small but adversarial perturbations. A common approach is to design  adversarial training algorithms by using adversarial examples as training data \citep{madry2017towards,tramer2018ensemble,shafahi2019adversarial}. 
Another line of works \citep{cohen2019certified,zhang2021towards} proposes some provably robust models to tackle this problem. 
However, while the state-of-the-art adversarial training methods can achieve high robust training accuracy (e.g. nearly $100\%$ on CIFAR-10 \citep{raghunathan2019adversarial}), all existing methods suffer from large robust test error. Therefore, it is natural to ask what is the cause for such a large generalization gap in the context of robust learning.

Previous works have studied 
the hardness of achieving adversarial robustness from different perspectives. 
A well-known phenomenon called the \textit{robustness-accuracy tradeoff} has been empirically observed \citep{raghunathan2019adversarial} and theoretically proven to occur in different settings \citep{tsipras2019robustness,zhang2019theoretically}. 
\citet{dohmatob2019generalized} shows that adversarial robustness is impossible to achieve under certain assumptions on the data distribution, while it is shown in
\citet{nakkiran2019adversarial} that even when an adversarial robust classifier does exist, it can be exponentially more complex than its non-robust counterpart. \citet{hassani2022curse} studies the role of over-parametrization on adversarial robustness by focusing on random features regression models.

At first glance, these works seem to provide convincing evidence that robustness is hard to achieve in general. 
However, this view is challenged by \citet{yang2020closer}, who observes that for real data sets, different classes tend to be well-separated (as defined below), while the perturbation radius is often much smaller than the separation distance. As pointed out by \citet{yang2020closer}, all aforementioned works fail to take this separability property of data into consideration.

\begin{definition}[Separated data] \label{separated_data}
Suppose that $A, B \subset \R^d$ and $\epsilon > 0$. We say that $A, B$ are $\epsilon$-separated under $\ell_p$ norm ($1\leq p\leq +\infty$) if
\begin{equation}
    \notag
    \left\|\bm{x}_A - \bm{x}_B \right\|_p \geq \epsilon,\quad \forall \bm{x}_A \in A, \bm{x}_B \in B.
\end{equation}
\end{definition}

Indeed, this assumption is necessary to ensure the existence of a robust classifier. Without this separated condition, it is clear that there is no robust classifier even if a non-robust classifier always exists, as discussed above. 


Recently, \citet{bubeck2021universal} shows that for regression problems, over-parameterization may be necessary for achieving robustness. However, they measure robustness of a model by its training error and Lipschitz constant, which has a subtle difference with \textit{robust test error} \citep{madry2017towards}; see the discussions in ~\cite[Section 1.1]{bubeck2021universal}.


To sum up the above, although existing robust training algorithms result in low robust test accuracy, previous works
do not provide a satisfactory explanation of this phenomenon, since there exists a gap between their settings and practice.
In particular, it is not known whether achieving robustness can be easier for data with additional structural properties such as separability \citep{yang2020closer} and low intrinsic dimensionality \citep{gong2019intrinsic}.

In this paper, we make an important step towards understanding robust generalization from the viewpoint of neural network expressivity. 
Focusing on binary classification problems with separated data (cf. Definition \ref{separated_data}) in $\R^d$, we make the following contributions:
\begin{itemize}
    \item Given a data set $\mathcal{D}$ of size $N$ that satisfies a separability condition, we show in Section~\ref{training} that it is possible for a ReLU network with $\Tilde{\mathcal{O}}(Nd)$ weights to robustly classify $\mathcal{D}$. In other words, an over-parameterized ReLU network with reasonable size can achieve $100\%$ robust training accuracy.
    \item 
    We next consider the robust test error (cf. Definition \ref{robust_generalization_error}). As a warm-up, we show in Section \ref{general} that, in contrast with the robust \textit{training} error, mere separability of data does not imply that low robust test error can be attained by neural networks, unless their size is exponential in $d$. This motivates the subsequent sections where we consider data with additional structures.
    \item In Section \ref{linear_separable_setting}, 
    we prove the main result of this paper, which states that for achieving low robust test error, an $\exp(\Omega(d))$ lower bound on the network size is inevitable, even when the underlying distributions of the two classes are linear separable. 
    Moreover, this lower bound holds for arbitrarily small perturbation radius and more general models as long as their VC dimension is at most polynomial in the number of parameters.
    \item 
    Finally, in Section \ref{low_dimensional_manifold_setting} we consider data that lies on a $k$-dimensional manifold ($k \ll d$), and prove an improved upper bound $\exp(\O(k))$ for the size of neural networks for achieving low robust test error. Nonetheless, the curse of dimensionality is inescapable -- the lower bound is also exponential in $k$.
\end{itemize}

The upper and lower bounds on network size are summarized in Table \ref{table}. Overall, our theoretical analysis suggests that the hardness of achieving robust generalization may stem from the expressive power of practical models.


\begin{table}[t]
    \setlength{\abovecaptionskip}{0.2cm}  
    \setlength{\belowcaptionskip}{0.2cm} 
    \caption{Summary of our main results.}
    \label{table}
    \centering
    \begin{tabular}{|c|c|c|c|c|}
    \hline \multirow{3}*{Params} & \multicolumn{4}{c|}{Setting}\\
    \cline{2 - 5}
    
          {}& \multirow{2}*{Robust Training} & \multicolumn{3}{c|}{Robust Generalization}  \\
    \cline{3-5}
    {} & {}& {General Case} & {Linear Separable} & {$k-$dim Manifold} \\
    \hline 
           \multirow{2}*{Upper Bound} & $\mathcal{O}(N d)$  & \multicolumn{2}{c|}{$\exp(\mathcal{O}(d))$} & $\exp(\mathcal{O}(k))$ \\
    {}&{(Thm \ref{robust_training_error})}&\multicolumn{2}{c|}{(Thm \ref{general_upper_bound})}&{(Thm \ref{manifold_upper_bound})} \\
    \hline 
    \multirow{2}*{Lower Bound} & $\Omega(\sqrt{N d})$ & $\exp(\Omega(d))$ & $\exp(\Omega(d))$ & $\exp(\Omega(k))$\\
     {}&{(Thm \ref{robust_training_lower_bound})
    }&{(Thm \ref{lower_bound_general})}&{(Thm \ref{linear_separable_lower_bound})}&{(Thm \ref{manifold_lower_bound})}\\
    \hline
    \end{tabular}
    
\end{table}

\subsection{Implications of our results}

Before moving on to technical parts, we would like to first discuss the implications of our results by comparing them to previous works.

Our main result is the exponential lower bound on the neural network size for generalization. First, different from previous hardness results, our result is established for data sets that have desirable structural properties, hence more closely related to practical settings. Note that the separability condition implies that there \textit{exists} a classifier that can perfectly and robustly classify the data i.e. achieve zero robust test error. However, we show that such classifier is hard to approximate using neural networks with moderate size.

Second, it is a popular belief that many real-world data sets are intrinsically low dimensional, although they lie in a high dimensional space. Our results imply that low dimensional structure makes robust generalization possible with a neural network with significantly smaller size (when $k \ll d$). However, the size must still be exponential in $k$.

Finally, we show that there exists an \textit{exponential} separation between the required size of neural networks for achieving low robust training and test error. Based on our results, we conjecture that the widely observed drop of robust test accuracy is not due to limitations of existing algorithms -- rather, it is a more fundamental issue originating from the expressive power of neural networks.

\subsection{Related works}

{\noindent \textbf{Robust test error.}} \citet{madry2017towards} proposes the notion of robust test error to measure the performance of a model under adversarial attacks. \citet{schmidt2018adversarially} shows that achieving low $\ell_\infty$ robust test error requires more data than the standard setting. Then, \citet{bhagoji2019lower,dan2020sharp} make extensions to the $\ell_p$ robustness setting, where $p \ge 1$. \citet{bhattacharjee2021sample} studies the sample complexity of robust linear classification on the separated data. These works and the references therein mainly consider the sample complexity. In contrast, we aim to figure out how many parameters are needed to ensure that a low robust test error can be attained.

{\noindent \textbf{Robust Generalization Gap.}
}One surprising behavior of deep learning is that over-parameterized neural networks can generalize well, which is also called \textit{benign overfitting} that deep models have not only the powerful memorization but a good performance for unseen data \citep{zhang2017understanding,belkin2019reconciling}. This surprising characteristics has been also studied from the theoretical perspective \citep{neyshabur2017exploring,petzka2021relative}. However, in contrast to the standard (non-robust) generalization, for the robust setting, \citet{rice2020overfitting} empirically investigates robust performance of models based on adversarial training methods, which are designed to improve adversarial robustness \citep{szegedy2013intriguing, madry2017towards}, and the work shows that \textit{robust overfitting} can be observed on multiple datasets. In this paper, we mainly focus on robust overfitting (robust generalization gap) theoretically, and provide a theoretical understanding of this phenomenon from the perspective of expressive ability.

{\noindent\textbf{Robust interpolation.}}
\citet{bubeck2021law} studies the smooth interpolation problem and proposes a conjecture that over-parameterization is necessary for this problem. Then \citet{bubeck2021universal} establishes the law of robustness under the isoperimetry data distribution assumption. Specifically, they prove an $\Tilde{\Omega}(\sqrt{N d/p})$ Lipschitzness lower bound for smooth interpolation, where $N,d$, and $p$ denote the sample size, the inputs' dimension, and the number of parameters, respectively. This result indicates that over-parameterization may be necessary for robust learning. Recently, \citet{zhang2022many} studies the smooth interpolation problem without the assumption that the data are drawn from an isoperimetry distribution and provides some results on how many data are needed for robust learning. This line of works focuses on the training error and the worst-case robustness (Lipschitz constant), but ignores robust generalization.
 


{\noindent \textbf{Memorization power of neural networks.}}
Our work is related to another line of works \citep[e.g.,][]{baum1988capabilities,yun2019small,bubeck2020network,zhang2021towards,rajput2021exponential, zhang2021expressivity, vardi2021optimal} on the memorization power of neural networks. Among these works, \citet{yun2019small} shows that a neural network with $\mathcal{O}(N)$ parameters can memorize the data set with zero error, where $N$ is the size of the data set. Under an additional separable assumption, \citet{vardi2021optimal} derives an improved upper bound of $\Tilde{\mathcal{O}}(\sqrt{N})$, which is shown to be optimal. In this work, we show that $\Tilde{\mathcal{O}}(Nd)$ parameters is sufficient for achieving low robust training error. This is in contrast with our exponential lower bound for low robust \textit{test} error.

{\noindent \textbf{Function approximation.}} Our work is also related to a vast body of works on function approximation via neural networks \citep[e.g.,][]{cybenko1989approximation,hornik1991approximation,lu2017expressive,yarotsky2017error,hanin2019universal}. \citet{yarotsky2017error} is mostly related, which shows that the functions in Sobolev spaces can be uniformly approximated by deep ReLU networks. On the other hand, the line of works is also related that studies using deep ReLU networks to approximate functions supported on low dimensional manifolds. See e.g., \citet{chui2018deep,shaham2018provable,chen2019efficient} and references therein. In particular, \citet{chen2019efficient} proves that any $C^{n}$ function in H\"{o}lder spaces can be $\epsilon-$approximated by the neural network with size $\mathcal{O}(\epsilon^{-k/n})$, where $k$ is the $\text{intrinsic}$ dimension of the manifold embedded in $\mathbb{R}^d$. In the robust classification scenario, we can also achieve dimensionality reduction by the low dimensional manifold assumption. Recently, \citet{liu2022benefits} studies robustness from the viewpoint of function approximation, and it shows benefits of over-parameterization in robust learning. Different from only focusing on smooth approximation in their paper, we prove not only upper bounds but lower bounds on network size to achieve low robust test error, which implies the inherent hardness of robust generalization.

\subsection{Notations}
Throughout this paper, we use $\left\|\cdot\right\|_p, p \in [1,+\infty]$ to denote the $\ell_p$ norm in the vector space $\R^d$. For $\bm{x} \in \R^d$ and $A \subset \R^d$, we can define the distance between $\bm{x}$ and $A$ as $d_p(\bm{x},A) = \inf\{ \|\bm{x}-\bm{y}\|_p: \bm{y} \in A\}$. For $r > 0$, $\B_p(\bm{x},r) = \left\{ \bm{y}\in\R^d: \left\|\bm{x}-\bm{y}\right\|_p \leq r\right\}$ is defined as the $\ell_p$-ball with radius $r$ centered at $\bm{x}$. For a function class $\mathcal{F}$, we use $d_{VC}(\mathcal{F})$ to denote its VC-dimension. A \textit{multilayer neural network} is a function from input $\bm{x} \in \R^d$ to output $\bm{y} \in \R^m$, recursively defined as follows:
\begin{equation}
    \notag
    \begin{aligned}
        \bm{h}_1 &= \sigma\left( \bm{W}_1 \bm{x} + \bm{b}_1 \right), \quad \bm{W}_1 \in \R^{m_1\times d}, \bm{b}_1 \in \R^{m_1}, \\
        \bm{h}_{\ell} &= \sigma\left( \bm{W}_{\ell} \bm{h}_{\ell-1} + \bm{b}_{\ell} \right), \quad \bm{W}_{\ell} \in \R^{m_{\ell}\times m_{\ell-1}}, \bm{b}_{\ell} \in \R^{m_{\ell}},2\leq \ell \leq L-1, \\
        \bm{y} &= \bm{W}_L \bm{h}_L + \bm{b}_L,\quad \bm{W}_L \in \R^{m\times m_L},\bm{b}_L \in \R^m,
    \end{aligned}
\end{equation}
where $\sigma$ is the activation function and $L$ is the depth of the neural network. In this paper, we mainly focus on ReLU networks i.e. $\sigma(x)=\max\{0,x\}$. The size of a neural network is defined as its number of weights/parameters i.e. the number of its non-zero connections between layers.

%% file: preprint_train.tex
\section{Mild Over-parameterized ReLU Nets Achieve Zero Robust Training Error }
\label{training}

With access to only finite amount of data, a common practice for learning a robust classifier is to minimize the \textit{robust training error}(defined below). In this section, we show that neural networks with reasonable size can achieve zero robust training error on a finite training set.

\begin{definition}[Robust training error]
Given a data set $\mathcal{D}=\left\{ (\bm{x}_i,y_i)\right\}_{1\leq i\leq N}$, $y_i \in \{-1,+1\}$ and an adversarial perturbation radius $\delta \geq 0$, the robust training error of a classifier $f$ is defined as
\begin{equation}
    \hat{\L}_{\mathcal{D}}^{p,\delta}(f) = \frac{1}{N} \sum_{i=1}^N \mathbb{I}\left\{ \exists \bm{x}' \in \B_{p}(\bm{x}_i;\delta), \mathrm{sgn}(f(\bm{x}'))\neq y_i \right\} . \nonumber
\end{equation}
\end{definition}

When $\delta = 0$, the definition coincides with the standard training error. In this paper, we mainly focus on the case $p=2$ and $p=\infty$, but our results can be extended to general $p$ as well.

The following is our main result in this section, which states that for binary classification problems, a neural network with $\Olog(Nd)$ weights can perfectly achieve robust classification on a data set of size $N$. The detailed proof is deferred to Appendix \ref{robust_training_error_appendix}.

\begin{theorem}
\label{robust_training_error}
Suppose that $\mathcal{D} \subset \B_{p}(0,1)$  with $p \in \{2,+\infty\}$ consists of $N$ data, 
and the two classes in $\mathcal{D}$ are $2\epsilon$-separated (cf. Definition \ref{separated_data}), where $\epsilon \in \left(0,\frac{1}{2}\right)$ is a constant. Let robustness radius $\delta < \frac{1}{2}\epsilon$, then there exists a classifier $f$ represented by a ReLU network with at most
\begin{equation}
    \notag
    \O\left( N d \log\left( \delta^{-1} d\right) + N\cdot \mathrm{polylog}(\delta^{-1}N)\right)
\end{equation}
parameters, such that $\hat{\L}_{\mathcal{D}}^{p,\delta}(f) = 0$.
\end{theorem}


Theorem \ref{robust_training_error} implies that neural networks is quite efficient for robust classification of finite training data. We also derive a lower bound in the same setting, which is stated below. It is an interesting future direction to study whether this lower bound can be achieved.



\begin{theorem}
\label{robust_training_lower_bound}
Let $p \in \{2,+\infty\}$ and $\mathcal{F}_n$ be the set of functions represented by ReLU networks with at most $n$ parameters. For arbitrary $2\epsilon$-separated data set $\mathcal{D}$ under $\ell_p$ norm, if there exists a classifier $f \in \mathcal{F}_n$ such that $\hat{\L}_{\mathcal{D}}^{p,\delta}(f) = 0$, then it must hold that $n = \Omega(\sqrt{N d})$.
\end{theorem}

The detailed proof of Theorem \ref{robust_training_lower_bound} is in Appendix \ref{robust_training_lower_bound_appendix}. We leave it as a future direction to study whether this lower bound can be attained. 
While optimal (non-robust) memorization of $N$ data points only needs constant width \citep{vardi2021optimal}, our construction in Theorem \ref{robust_training_error} has width $\Olog(Nd)$. Therefore, if our upper bound is tight, then Theorem \ref{robust_training_error} can probably explain why increasing the network width can benefit robust training \citep{madry2017towards}.

%% file: preprint_general.tex
\section{Hardness of Robust Generalization : A Warm-up}
\label{general}

In the previous section, we give an upper bound on the size of ReLU networks to robustly classify finite training data. 
However, it says nothing about \textit{generalization}, or the robust test error, which is arguably a crucial aspect of evaluating the performance of a trained model. 
As a warm-up, in this section we first consider the same setting as Section \ref{training} where we only assume the data to be well-separated. We show that in this setting, even achieving high standard test accuracy requires exponentially large neural networks in the worst case, which is quite different from empirical observations. This motivates to consider data with additional structures in subsequent sections.

\begin{definition}[Robust test error]
\label{robust_generalization_error}
Given a probability measure $P$ on $\R^d \times \{-1,+1\}$ and a robust radius $\delta \geq 0$, the robust test error of a classifier $f:\R^d \to \R$ w.r.t $P$ and $\delta$ under $\ell_p$ norm is defined as
\begin{equation}
    \L_{P}^{p,\delta}(f) = \mathbb{E}_{(\bm{x},y) \sim P}\left[\max_{\|\bm{x}'-\bm{x}\|_{p}\leq \delta}\mathbb{I}\{ \operatorname{sgn}(f(\bm{x}')) \ne y \}\right] . \nonumber
\end{equation}
\end{definition}
In contrast with the training set which only consists of finite data points, when studying generalization, we must consider potentially infinite points in the classes that we need to classify. As a result, we consider two disjoint sets $A, B \in [0,1]^d$, where points in $A$ have label $+1$ and points in $B$ have label $-1$. We are interested in the following questions:
\begin{itemize}
    \item Does there exists a robust classifier of $A$ and $B$?
    \item If so, can we derive upper and lower bounds on the size of a neural network to robustly classify $A$ and $B$?
\end{itemize}

It turns out that, similar to the previous section, the $\epsilon$-separated condition (cf. Definition \ref{separated_data}) ensures the existence of such a classifier. Moreover, it can be realized by a Lipschitz function. This fact has been observed in \citet{yang2020closer}, and we provide a different version of their result below for completeness.

\begin{proposition} \label{prop:distanse:classifier}
For $2\epsilon$-separated $A,B \subset [0,1]^{d}$ under $\ell_p$ norm with $p \in \{2,+\infty\}$, the classifier $f^{*}(\bm{x}):=\frac{d_{p}(\bm{x},B)-d_{p}(\bm{x},A)}{d_{p}(\bm{x},A)+d_{p}(\bm{x},B)}$
is $\epsilon^{-1}$-Lipschitz continuous, and satisfies $\L_{P}^{p,\eps}(f^*) = 0$ for any probability distribution $P$ on $A \cup B$.
\end{proposition}

Based on this observation, \citet{yang2020closer} concludes that adversarial training is not inherently hard. Rather, they argue that current pessimistic results on robust test error is due to the limits of existing algorithms. However, it remains unclear whether the Lipschitz function constructed in Proposition \ref{prop:distanse:classifier} can actually be efficiently approximated by neural networks. The following theorem shows that ReLU networks with exponential size is sufficient for as robust classification.

\begin{theorem}
\label{general_upper_bound}
For any two $2\epsilon$-separated $A,B \subset [0,1]^{d}$ under $\ell_{p}$ norm with $p \in \{2,+\infty\}$, distribution $P$ on the supporting set $S = A \cup B$ and robust radius $c \in (0,1)$, there exists a ReLU network $f$ with at most $\Tilde{\mathcal{O}}\left(((1-c)\epsilon)^{-d}\right)$
parameters, such that $\L_{P}^{p,c\epsilon}(f) = 0$.
\end{theorem}
The detailed proof is deferred to Appendix \ref{appendix:pf:general:ub}.
Indeed, it is well known that without additional assumptions, an exponentially large number of parameters is also \textit{necessary} for approximating a Lipschitz function \citep{devore1989optimal,shen2022optimal}. 
This result motivates us to consider the second question listed above. The following result implies that even \textit{without} requiring robustness, neural networks need to be exponentially large to correctly classify $A$ and $B$:

\begin{theorem}
\label{lower_bound_general}
Let $\mathcal{F}_n$ be the set of functions represented by ReLU networks with at most $n$ parameters. Suppose that for any $2\epsilon$-separated sets $A, B \subset [0,1]^d$ under $\ell_{p}$ norm with $p \in \{2,+\infty\}$, there exists $f \in \mathcal{F}_n$ that can classify $A,B$ with zero (standard) test error, then it must hold that $n = \Omega\left((2\epsilon)^{-\frac{d}{2}}\left(d \log\left(1/2\epsilon\right)\right)^{-\frac{1}{2}}\right)$.
\end{theorem}



Theorem \ref{lower_bound_general} implies that mere separability of data sets 
is insufficient to guarantee that they can be classified by ReLU networks, unless the network size is exponentially large. The detailed proof is in Appendix \ref{appendix:pf:general:lb}. 

However, one should be careful when interpreting the conclusion of Theorem \ref{lower_bound_general}, since real-world data sets may possess additional structural properties. Theorem \ref{lower_bound_general} does not take these properties into consideration, so it does not rule out the possibility that these additional properties make robust classification possible. Specifically, the joint distribution of data can be decomposed as
\begin{equation}
    \mathcal{P}(X,Y) = \underbrace{\mathcal{P}(Y \mid X)}_{\text{labeling mapping}} \underbrace{\mathcal{P}(X)}_{\text{input}}, \nonumber
\end{equation}
where $\mathcal{P}(X,Y),\mathcal{P}(Y \mid X)$, and $\mathcal{P}(X)$ denote the joint, conditional and marginal distributions, respectively. In subsequent sections, we consider two well-known properties of data sets that correspond to the labeling mapping structure (Section \ref{linear_separable_setting}) and the input structure (Section \ref{low_dimensional_manifold_setting}), respectively, and study whether they can bring improvement to neural networks' efficiency for robust classification.


%% file: preprint_linear.tex
\section{Robust Generalization for Linear Separable Data}
\label{linear_separable_setting}

We have seen that for separated data, if no other structural properties are taken into consideration, even standard generalization requires exponentially large neural networks. However, in practice it is often possible to train neural networks that can achieve fairly high standard test accuracy, indicating a gap between the setting of Section \ref{general} and practice. 

This motivates us to consider the following question: assuming that there exists a simple classifier that achieves zero standard test error on the data, is it guaranteed that neural networks with reasonable size can also achieve high \textit{robust} test accuracy?

We give a negative answer to this question. Namely, we show that even in the arguably simplest setting where the given data is linear separable and well-separated (cf. Definition \ref{separated_data}), ReLU networks still need to be exponentially large to achieve high robust test accuracy. 

\subsection{Main results under the linear separable assumption}

Clearly, the robust test error (cf. Definition \ref{robust_generalization_error}) depends on the underlying distribution~$P$. We consider a class of data distributions which have bounded density ratio with the uniform distribution:

\begin{definition}[Balanced distribution]
\label{balanced_dist}
Let $S \subset \R^n$ such that there exists a uniform probability measure $m_0$ on $S$. A distribution $P$ on $S$ is called $\mu$-balanced if
\begin{equation}
    \notag
    \inf\left\{ \frac{P(E)}{m_0(E)}: E \text{ is Lebesgue measurable and } m_0(E) > 0\right\} \geq \mu.
\end{equation}
\end{definition}

\begin{remark}
Definition \ref{balanced_dist} has also appeared in ~\cite[Theorem 1]{shafahi2018adversarial}, which gives an impossibility result on robust classification, albeit in a completely different setting. Intuitively, it rules out the possibility that data points in certain regions are heavily under-represented.
\end{remark}


The following theorem is the main result of this paper, and the proof sketch is deferred to Section~\ref{subsec:proof_sketch}.

\begin{theorem}
\label{linear_separable_lower_bound}
Let $\epsilon \in (0,1)$ be a small constant, $p \in \{2,+\infty\}$ and $\mathcal{F}_n$ be the set of functions represented by ReLU networks with at most $n$ parameters. There exists a sequence $N_d = \Omega\left((2\epsilon)^{-\frac{d-1}{6}}\right), d \geq 1$ and a universal constant $C_1 > 0$ such that the following holds: for any $c \in (0,1)$, there exists two linear separable sets $A,B \subset [0,1]^{d}$ that are $2\epsilon$-separated under $\ell_p$ norm, such that for any $\mu_{0}$-balanced distribution $P$ 
on the supporting set $S = A \cup B$ and robust radius $c\epsilon$ we have
\begin{equation}
     \inf\left\{ \L_{P}^{p,c\epsilon}(f) : f \in \mathcal{F}_{N_d}\right\} \geq C_1 \mu_0.
    \nonumber
\end{equation}
\end{theorem}


Theorem \ref{linear_separable_lower_bound} states that the robust test error is lower-bounded by a positive constant $\alpha = C_1 \mu_0$ unless the ReLU network has size larger than $\exp(\Omega(d))$.
On the contrary, if we do not require robustness, then the data can be classified by a simple linear function. Moreover, this classifier can be learned with a poly-time efficient algorithm (The detailed proof is in Appendix \ref{learn_linear_separable_appendix}) :

\begin{theorem}
\label{learn_linear_separable}
For any two linear-separable $A,B \subset [0,1]^{d}$, a distribution $P$ on the supporting set $S= A \cup B$, $\delta > 0$ and $\beta > 0$, let $H$ be the family of $d-$dimensional hyperplane classifiers. Then, there exists a poly-time efficient algorithm $\mathcal{A}: 2^{S} \rightarrow H$, for $N= \Omega(d/\beta^{2})$ training instances independently randomly sampled from $P$, with probability $1-\delta$ over samples, we can use the algorithm $\mathcal{A}$ to learn a classifier $\hat{f} \in F$ such that
\begin{equation}
    \notag
    \L_{P}(\hat{f}) \leq \beta,
\end{equation}
where $\L_{P}(f):= \mathbb{P}_{(\bm{x},y) \sim P}\{y \ne f(\bm{x})\}$ denotes the standard test error.
\end{theorem}

The practical implication of Theorem \ref{linear_separable_lower_bound} is two-fold. First, by comparing with Theorem \ref{learn_linear_separable}, one can conclude that robust classification may require exponentially more parameters than the non-robust case, which is consistent with the common practice that larger models are used for adversarial robust training. Second, together with our upper bound in Theorem~\ref{robust_training_error}, Theorem \ref{linear_separable_lower_bound} implies an \textit{exponential} separation of neural network size for achieving high robust training and test accuracy.

\subsection{Exponential lower bound for more general models}
In general, our lower bounds hold true for a variety of neural network families and other function classes as well, as long as their VC dimension is at most polynomial in the number of parameters, which is formally stated as Theorem \ref{general_linear_separable_lower_bound} that can be derived by the proof of Theorem \ref{linear_separable_lower_bound} directly.

\begin{theorem}
\label{general_linear_separable_lower_bound}
Let $\epsilon \in (0,1)$ be a small constant, $p \in \{2,+\infty\}$ and $\mathcal{G}_n$ be the family of parameterized models with at most $n$ parameters, satisfying the VC dimension of function family $\text{VC-dim}(\mathcal{G}_n)$ is at most $\operatorname{poly}(n)$. Then, there exists a sequence $N_d = \operatorname{exp}(\Omega(d)), d \geq 1$ and a universal constant $C_1' > 0$ such that the following holds: for any $c \in (0,1)$, there exists two linear separable sets $A,B \subset [0,1]^{d}$ that are $2\epsilon$-separated under $\ell_p$ norm, such that for any $\mu_{0}$-balanced distribution $P$ 
on the supporting set $S = A \cup B$ and robust radius $c\epsilon$ we have
\begin{equation}
     \inf\left\{ \L_{P}^{p,c\epsilon}(g) : g \in \mathcal{G}_{N_d}\right\} \geq C_1' \mu_0.
    \nonumber
\end{equation}
\end{theorem}
In other words, the robust generalization error cannot be lower that a constant $\alpha = C_1'\mu_0$ unless the model, satisfying the property of their VC dimension polynomially bounded by the number of parameters, has exponential larger size. Indeed, this property is satisfied for e.g. feedforward neural networks with sigmoid \citep{karpinski1995polynomial} and piecewise polynomial \citep{bartlett2019nearly} activation functions. Therefore, our results reveal that the hardness of robust generalization may stem from the expressive power of generally practical models.

\subsection{Proof sketch of Theorem \ref{linear_separable_lower_bound}}
\label{subsec:proof_sketch}
In this subsection, we present a proof sketch for Theorem \ref{linear_separable_lower_bound} in the $\ell_{\infty}$-norm case. We only consider $P$ to be the uniform distribution, extending to $\mu_0$-balanced distributions is not difficult, The case of $\ell_2$-norm is similar and can be found in the Appendix.

\begin{proof}[Proof Sketch]
Let $K = \floor{\frac{1}{2\eps}}$, and $\phi: \{1,2,\cdots,K\}^{d-1} \to \left\{ -1,+1\right\}$ be an arbitrary mapping, we define $S_{\phi} = \left\{ \left( \frac{i_1}{K}, \frac{i_2}{K}, \cdots, \frac{i_{d-1}}{K}, \frac{1}{2}+ \epsilon_0\cdot\phi(i_1,i_2,\cdots,i_{d-1}) \right) : 1 \leq i_1,i_2,\cdots, i_{d-1} \leq K \right\}$, where $\epsilon_0$ is an arbitrarily small constant. The hyperplane $x^{(d)} = \frac{1}{2}$ partitions $S_{\phi}$ into two subsets, which we denote by $A_{\phi}$ and $B_{\phi}$. It is not difficult to check that $A_{\phi}$ and $B_{\phi}$ satisfies all the required conditions.


\begin{figure}[h]
    \centering
    \includegraphics[width=0.7\linewidth]{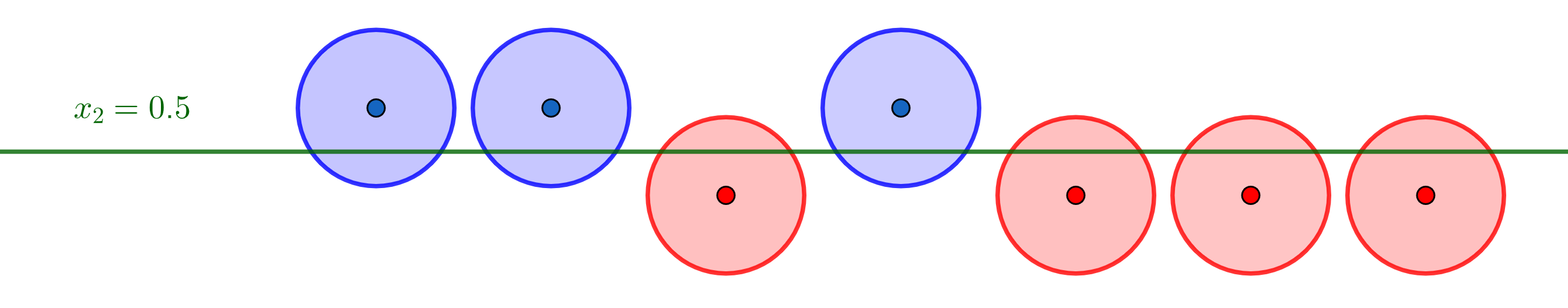}
    \caption{An example of our construction for $d=2$. We choose $A, B$ as the set of blue points and red points, respectively.}
    \label{linear_separable_fig}
\end{figure}

Our goal is to show that there exists some choice of $\phi$ such that robust classification is hard. To begin with, suppose that robust classification with accuracy $1-\alpha$ can be achieved with at most $M$ parameters for all $\phi$, then these networks can all be embedded into an \textit{enveloping network} $F_{\bm{\theta}}$ of size $\O(M^3)$.

Define $\Tilde{S} = \left\{ \left( \frac{i_1}{K}, \frac{i_2}{K}, \cdots, \frac{i_{d-1}}{K}, \frac{1}{2} \right) : 1 \leq i_1,i_2,\cdots, i_{d-1} \leq K \right\}$. Robustness implies that for all possible label assignment to $\Tilde{S}$, at least $(1-\alpha)K^{d-1}$ points can be correctly classified by $F_{\bm{\theta}}$.

If $\alpha = 0$ i.e. perfect classification is required, then we can see that the $\text{VC-dim}(F_{\bm{\theta}}) \geq K^{d-1}$ , which implies that its size must be exponential, by applying classical VC-dimension upper bounds of neural networks \citep{bartlett2019nearly}.

When $\alpha > 0$, we cannot directly use the bound on VC-dimension. Instead, we use a double-counting argument to lower-bound the \textit{growth function} of some subset of $\Tilde{S}$. 

Let $V = \frac{1}{2}K^{d-1}$. Each choice of $\phi$ corresponds to $\binom{(1-\alpha)K^{d-1}}{V}$ labelled $V$-subset of $\Tilde{S}$ that are correctly classified. There are a total of $2^{K^{d-1}}$ choices of $\phi$, while each labelled $V$-subset can be obtained by at most $2^{K^{d-1}-V}$ different $\phi$. As a result, the total number of labelled $V$-subset correctly classified by $F_{\bm{\theta}}$ is at least $2^V \binom{(1-\alpha)K^{d-1}}{V}$.

On the other hand, the total number of $V$-subset of $\widetilde{S}$ is $\binom{K^{d-1}}{V}$, thus there must exists a $V$-subset $\mathcal{V}_0 \subset \widetilde{S}$, such that at least
\begin{equation}
    \label{growth_func_bound}
    \binom{K^{d-1}}{V}^{-1}\cdot 2^V \binom{(1-\alpha)K^{d-1}}{V} \geq \left( \frac{2\left( (1-\alpha)K^{d-1}-V\right)}{K^{d-1}-V}\right)^V 
    \geq C_{\alpha}^{K^{d-1}}
\end{equation}
different labellings of $\mathcal{V}_0$ are correctly classified by $F_{\bm{\theta}}$, where $C_{\alpha} = \sqrt{2(1-2\alpha)} > 1$ for $\alpha = 0.1$. Since \eqref{growth_func_bound} provides a lower bound for the growth function, together with the upper bound of growth function in terms of VC-dimension, we can deduce that $\text{VC-dim}(F_{\bm{\theta}}) \geq 0.05 K^{d-1}$. Finally, the conclusion follows by applying the VC-dimension bounds in \citet{bartlett2019nearly}.
\end{proof}
\begin{remark}
The connection between VC dimension and approximation error has been explored in a number of previous works \citep{yarotsky2017error,shen2022optimal} to provide lower bounds on the network size for approximating a given function class. Here we consider the problem of robust classification which is of more practical interest than function approximation, and our main technical contribution is an exponential lower bound on the VC dimension. Our proof formalizes the folklore that adversarial training is hard since it requires a more complicated decision boundary. We note that similar ideas have been used to show benefits of depth in neural networks \citep{telgarsky2016benefits,liang2017deep} but their techniques are restricted to one-dimensional functions.
\end{remark}


%% file: preprint_manifold.tex
\section{Robust Generalization for Low-Dimensional-Manifold Data}
\label{low_dimensional_manifold_setting}

In this section, we focus on refined structure of data's input distribution $\mathcal{P}(X)$. A common belief of real-life data such as images is that the data points lie on a low-dimensional manifold. It promotes a series of methods that are invented to make the dimensionality reduction, including linear dimensionality reduction (e.g., PCA \citep{pearson1901liii}) and non-linear dimensionality reduction (e.g., $t$-SNE \citep{hinton2008visualizing}). Several works have also empirically verified the belief. \citet{roweis2000nonlinear} and \citet{tenenbaum2000global} have demonstrated that image, speech and other variant form data can be modeled nearly on low-dimensional manifolds. In particular, \citet{wang2016auto} studies auto-encoder based dimensionality reduction, and shows that the $28\times28=784$ dimensional image from MNIST can be reduced to nearly $10$ dimensional representations, which corresponds to the intrinsic dimension of the handwritten digital dataset. 

Motivated by these findings, in this section, we assume that data lies on a low-dimensional manifold $\mathcal{M}$ embedded in $[0,1]^{d}$ i.e. $\operatorname{supp}(X) \subset \mathcal{M} \subset [0,1]^{d}$. We will show a improved upper bound that is exponential in the intrinsic $k$ of the manifold $\mathcal{M}$ instead of the ambient data dimension $d$ for the size of networks achieving zero robust test error, which implies the efficiency of robust classification under the low-dimensional manifold assumption. Also, we point out that the exponential dependence of $k$ is not improvable by establishing a matching lower bound.

\subsection{Preliminaries}
Let $\mathcal{M}$ be a $k-$dimensional compact Riemannian manifold embedded in $\mathbb{R}^{d}$, where $k$ is the intrinsic dimension ($k \ll d$). 

\begin{definition}[Chart, atlas and smooth manifold]
A chart for $\mathcal{M}$ is a pair $(U, \phi)$ such that $U \subset M$ is open and $\phi: U \rightarrow \mathbb{R}^{k}$, where $\phi$ is a homeomorphism; An atlas for $M$ is a collection $\left\{\left(U_{\alpha}, \phi_{\alpha}\right)\right\}_{\alpha \in A}$ of pairwise $C^{n}$ compatible charts such that $\bigcup_{\alpha \in A} U_{\alpha}=\mathcal{M}$; And we call $\mathcal{M}$ a smooth manifold if and only if $\mathcal{M}$ has a $C^{\infty}$ atlas.
\end{definition}


\begin{definition}[Partition of unity]
A $C^{\infty}$ partition of unity on a manifold $\mathcal{M}$ is a collection of non-negative $C^{\infty}$ functions $\rho_{\alpha}: \mathcal{M} \rightarrow \mathbb{R}_{+}$for $\alpha \in A$ that satisfy $(1)$ the collection of supports, $\left\{\operatorname{supp}\left(\rho_{\alpha}\right)\right\}_{\alpha \in A}$, is locally finite; and $(2)$ $\sum_{\alpha \in A} \rho_{\alpha}=1$.
\end{definition}

\begin{definition}[Poly-Partitionable]
We call that $\mathcal{M}$ is poly-partitionable if and only if, for a tangent-space-induced atlas $\{(U_{\alpha},T_{\alpha})\}_{\alpha \in A}$ of $\mathcal{M}$, there exists a particular partition of unity $\{\rho_{\alpha}\}_{\alpha \in A}$ that satisfies $\rho_{\alpha} \circ T_{\alpha}^{-1}$ is a simple piecewise polynomial in $\mathbb{R}^{k}$, where simple piecewise polynomial is defined as the composite mapping between a polynomial and a size-bounded ReLU network.
\end{definition}

The concept, poly-partitionable, defines a class of manifolds that have simple partition of unity, which is a generalization of some structures in the standard Euclidean space $\mathbb{R}^{d}$. For example, an explicit construction for low-dimensional manifold $[0,1]^{k}$ is $\{\phi_{m}(x)\}$ in \citet{yarotsky2017error}, where the coordinate system is identity mapping.

\subsection{Main results under the low-dimensional manifold assumption}
Before giving our main results, we first extend robust classification to the version of manifold.

\begin{definition}[Robust classification on a manifold]
Given a probability measure $P$ on $\mathcal{M} \times \{-1,+1\}$ and a robust radius $\delta$, the robust test error of a classifier $f:\mathcal{M} \to \R$ w.r.t $P$ and $\delta$ under $\ell_p$ norm is defined as $\L_{\mathcal{M},P}^{p,\delta}(f) = \mathbb{E}_{(x,y) \sim P}\left[ \max_{x' \in \mathcal{M},\|x'-x\|_{p}\leq \delta}\mathbb{I}\{ \operatorname{sgn}(f(x')) \ne y\}\right] .$
\end{definition}

Now, we present our main result in this section, which establishes an improved upper bound for size that is mainly exponential in the intrinsic dimension $k$ instead of the ambient data dimension $d$.

\begin{theorem}
\label{manifold_upper_bound}
Let $\mathcal{M} \subset [0,1]^{d}$ be a $k-$dimensional compact poly-partitionable Riemannian manifold with the condition number $\tau > 0$. For any two $2\epsilon-$ separated $A,B \subset \mathcal{M}$ under $\ell_{\infty}$ norm, distribution $P$ on the supporting set $S = A \cup B$ and robust radius $c \in (0,1)$, there exists a ReLU network $f$ with at most
\begin{equation}
    \Tilde{\mathcal{O}}\left(\left(\left(1-c\right)\epsilon/\sqrt{d}\right)^{-\Tilde{k}}\right) \nonumber
\end{equation}
parameters, such that $\L_{\mathcal{M},P}^{\infty,c\epsilon}(f) = 0$, where $\Tilde{k} = \mathcal{O}\left(k \log d\right)$ is almost linear with respect to the intrinsic dimension $k$, only up to a logarithmic factor.
\end{theorem}

\begin{proof}[Proof sketch]
The proof idea of Theorem \ref{manifold_upper_bound} has two steps. First, we construct a Lipschitz robust classifier $f^{*}$ in Proposition \ref{prop:distanse:classifier}. Then, we regard $f^{*}$ as the target function and use a ReLU network $f$ to approximate it on the manifold $\mathcal{M}$. The following lemma is the key technique that shows we can approximate Lipschitz functions on a manifold by using ReLU networks efficiently.

\begin{lemma}
\label{manifold_appr_lemma}
Let $\mathcal{M} \subset [0,1]^{d}$ be a $k-$dimensional compact poly-partitionable Riemannian manifold with the condition number $\tau > 0$. For any small $\delta > 0$ and a $L-$lipschitz function $g:\mathcal{M} \rightarrow \mathbb{R}$, there exists a function $\Tilde{g}$ implemented by ReLU network with at most $\Tilde{\mathcal{O}}\left((\sqrt{d}L/\delta)^{-\Tilde{k}}\right)$ parameters, such that $|g-\Tilde{g}| < \delta$ for any $x \in \mathcal{M}$, where $\Tilde{k}$ is the same as Theorem \ref{manifold_upper_bound}.
\end{lemma}

By applying the conclusion of Lemma \ref{manifold_appr_lemma}, we can approximate the $1/\epsilon-$Lipschitz function $f^{*}$ in Proposition \ref{prop:distanse:classifier} via a ReLU network $f$ with at most $\Tilde{\mathcal{O}}\left(\exp(\Tilde{k})\right)$ parameters, such that the uniform approximation error $\|f-f^{*}\|_{\ell_{\infty}(\mathcal{M})}$ at most $1-c$. 

Next, we prove the theorem by contradiction. Assume that there exists some perturbed input $x'$ that is mis-classified and the original input $x$ is in $A$. So we know $f(x') < 0$ and $f^{*}(x) < \epsilon '$. This impiles $ d_{\infty}(x',A) < d_{\infty}(x',B) < \frac{1+\epsilon '}{1-\epsilon '}d_{\infty}(x',A)$. Combined with $d_{\infty}(x',A)+d_{\infty}(x',B) \geq d_{\infty}(A,B) \geq 2\epsilon$, we have $d_{\infty}(x',A) > (1-\epsilon ')\epsilon = c \epsilon$, which is a contradiction.
\end{proof}

\begin{remark}
\label{manifold_appr_remark}
\citet{chen2019efficient} studies network-based approximation on smooth manifolds, and also establishes an $\mathcal{O}(\delta^{-k})$ bound for the network's size. However, different from their setting where the approximation error $\delta$ goes to zero, it is reasonable that the separated distance $\epsilon$ and robust radius $c$ are constants in our setting. If we simply follow their proofs, we can only obtain the bound  $\mathcal{O}((\delta/\mathcal{C}_\mathcal{M})^{-k})$ where $\mathcal{C}_\mathcal{M}$ also grows exponentially with respect to $k$, which further implies that the final result will be roughly $\exp (\mathcal{O}(k^{2}))$. This bound is too loose, especially when $k \approx \sqrt{d}$. To this end, we propose a novel approximation framework so as to improve the bound to $\exp(\mathcal{O}(k))$, which is presented as Lemma~\ref{manifold_appr_lemma}. And the detailed proof of Lemma \ref{manifold_appr_lemma} is deferred to Appendix \ref{manifold_appr_lemma_appendix}.
\end{remark}

Although we have shown that robust classification will be more efficient when data lies on a low-dimensional manifold, there is also a curse of dimensionality, i.e., the upper bound for the network's size is almost exponential in the intrinsic dimension $k$. The following result shows that the curse of dimension is also inevitable under the low-dimensional manifold assumption.

\begin{theorem}
\label{manifold_lower_bound}
Let $\epsilon \in (0,1)$ be a small constant. There exists a sequence $\{N_k\}_{k \geq 1}$ that satisfies $N_k = \Omega\left((2\epsilon\sqrt{d/k})^{-\frac{k}{2}}\right)$.
and a universal constant $C_1 > 0$ such that the following holds: let $\mathcal{M} \subset [0,1]^{d}$ be a complete and compact $k-$dimensional Riemannian manifold with non-negative Ricci curvature , then there exists two $2\epsilon$-separated sets $A,B \subset \mathcal{M}$ under $\ell_{\infty}$ norm, such that for any $\mu_{0}-$balanced distribution $P$ on the supporting set $S = A \cup B$ and robust radius $c \in (0,1)$, we have
\begin{equation}
     \inf\left\{ \L_{P}^{\infty,c\epsilon}(f) : f \in F_{N_k}\right\} \geq C_1 \mu_0.
    \nonumber
\end{equation}
\end{theorem}

In other words, the robust test error is lower-bounded by a positive constant $\alpha = C_1 \mu_0$ unless the neural network has size larger than $\exp (\Omega(k))$. The detailed proof of Theorem \ref{manifold_lower_bound} is presented in Appendix \ref{manifold_lower_bound_appendix}.

%% file: conclusion.tex
\section{Conclusion}
This paper provides a new theoretical understanding of the gap between the robust training and generalization error. We show that the ReLU networks with reasonable size can robustly classify the finite training data. On the contrary, even with the linear separable and well-separated assumptions, ReLU networks must be exponentially large to achieve low robust generalization error. Finally, we consider the scenario where the data lies on the low dimensional manifold and prove that the ReLU network, with a size exponentially in the intrinsic dimension instead of the inputs' dimension, is sufficient for obtaining low robust generalization error.  We believe our work opens up many interesting directions for future work, such as the tight bounds for the robust classification problem, or the reasonable assumptions that permit the polynomial-size ReLU networks to achieve low robust generalization error. 

%% file: appendix.tex
\begin{appendix}

\section{Preliminaries}

In this section, we recall some standard concepts and results in statistical learning theory.

\begin{definition}[growth function]
Let $\mathcal{F}$ be a class of functions from $\mathcal{X} \subset \R^d$ to $\{-1,+1\}$. For any integer $m \geq 0$, we define the growth function of $\mathcal{F}$ to be
\begin{equation}
    \notag
    \Pi_{\mathcal{F}}(m) = \max_{x_i\in\mathcal{X},1\leq i\leq m}\left|\left\{ (f(x_1),f(x_2),\cdots,f(x_m)) : f \in\mathcal{F}\right\}\right|.
\end{equation}
In particular, if $\left|\left\{ (f(x_1),f(x_2),\cdots,f(x_m)) : f \in\mathcal{F}\right\}\right| = 2^m$, then $(x_1,x_2,\cdots,x_m)$ is said to be \textit{shattered} by $\mathcal{F}$.
\end{definition}

\begin{definition}[Vapnik-Chervonenkis dimension]
Let $\mathcal{F}$ be a class of functions from $\mathcal{X} \subset \R^d$ to $\{-1,+1\}$. The VC-dimension of $\mathcal{F}$, denoted by $\text{VC-dim}(\mathcal{F})$, is defined as the largest integer $m \geq 0$ such that $\Pi_{\mathcal{F}}(m) = 2^m$. For real-value function class $\mathcal{H}$, we define $\text{VC-dim}(\mathcal{H}) := \text{VC-dim}(\mathrm{sgn}(\mathcal{H}))$.
\end{definition}

The following result gives a nearly-tight upper bound on the VC-dimension of neural networks.

\begin{lemma}
~\cite[Theorem 6]{bartlett2019nearly}
\label{lem:vc_dim}
Consider a ReLU network with $L$ layers and $W$ total parameters. Let
$F$ be the set of (real-valued) functions computed by this network. Then we have $\text{VC-dim}(F) = O(W\log(WL))$.
\end{lemma}

The growth function is connected to the VC-dimension via the following lemma; see e.g. ~\cite[Theorem 7.6]{anthony1999neural}.

\begin{lemma}
\label{growth_func_upper_bound}
Suppose that $\text{VC-dim}(\mathcal{F}) = k$, then $\Pi_m(\mathcal{F}) \leq \sum_{i=0}^k \binom{m}{i}$. In particular, we have $\Pi_m(\mathcal{F}) \leq \left( em/k \right)^k$ for all $m > k+1$.
\end{lemma}

\begin{lemma}
~\cite[Corollary 3.4]{mohri2018foundations}
\label{lem:vc_bound}
Let $H$ be a family of functions taking values in $\{-1,+1\}$ with $V C$-dimension $k$. Then, for any $\delta>0$, with probability at least $1-\delta$ over $m-$samples training dataset $S$ i.i.d. drawn from the data distribution $D$, the following holds for all $h \in H$ :
$$
\L_{D}(h) \leq \L_{S}(h)+\sqrt{\frac{2 k \log \frac{e m}{k}}{m}}+\sqrt{\frac{\log \frac{1}{\delta}}{2 m}},
$$
where $\L_{D}(h)$ and $\L_{S}(h)$ denote the standard test error and training error, respectively.
\end{lemma}

For deriving upper and lower bounds in the context of $\ell_2$-robustness, we also need to introduce the following concepts.

\begin{definition}[$\epsilon$-covering]
Given a set $\Theta \subset \R^d$, we say that $X = \left\{ \bm{x}_1,\bm{x}_2,\cdots,\bm{x}_n\right\} \subset \Theta$ is a $\delta$-covering of $\Theta$ if $\Theta \subset \cup_{i=1}^n \B_2(\bm{x}_i,\delta)$. The \textit{covering number} $\mathcal{C}(\Theta,\delta)$ is defined as the minimal size of a $\delta$-covering set of $\Theta$.
\end{definition}

The following proposition is straightforward from the definition.

\begin{proposition}
\label{covering_lower_bound}
Let $\Theta \subset \R^d$ has volume (i.e. Lebesgue measure) $V$, then
\begin{equation}
    \notag
    \mathcal{C}(\Theta,\delta) \geq v_d\cdot \delta^{-d} V,
\end{equation}
where $v_d$ is the volume of a $d$-dimensional unit ball.
\end{proposition}

\begin{definition}[$\epsilon$-packing]
Given a set $\Theta \subset \R^d$, we say that $X = \left\{ \bm{x}_1,\bm{x}_2,\cdots,\bm{x}_n\right\} \subset \Theta$ is a $\delta$-packing of $\Theta$ if $\left\|\bm{x}_i-\bm{x}_j\right\|_2 \geq \delta, \forall i\neq j$. The \textit{packing number} $\mathcal{P}(\Theta,\delta)$ is defined as the maximal size of a $\delta$-packing set of $\Theta$.
\end{definition}

The relationship between the covering and packing number is given by the following result. For completeness, we also provide a simple proof.

\begin{proposition}
\label{prop:packing_covering}
For any $\delta \geq 0$, we have $\mathcal{P}(\Theta,\delta) \geq \mathcal{C}(\Theta,\delta)$.
\end{proposition}

\begin{proof}
Consider a maximal packing $X = \{\bm{x}_1,\bm{x}_2,\cdots,\bm{x}_n\}$. Pick any $\bm{x} \in \Theta$, then there must exists some $\bm{x}_i \in X$ such that $\left\|\bm{x}-\bm{x}_i\right\|_2 \leq \delta$; otherwise, $X \cup \{\bm{x}\}$ is a larger packing set, which contradicts the definition of $X$.

Hence it must holds that $\Theta \subset \cup_{i=1}^n \B_2(\bm{x}_i,\delta)$ i.e. $X$ is a $\delta$-covering of $\Theta$. The conclusion follows.
\end{proof}

\section{Proofs for Section \ref{training}}

To prove Theorem \ref{robust_training_error}, we first recall some well-known results of neural networks for approximating simple functions.

\begin{lemma}
\label{simp_func}
Let $\eps > 0$, $0 < a < b$ and $B \geq 1$ be given.
\begin{enumerate}[(1).]
    \item ~\citep[Proposition 3]{yarotsky2017error} There exists a function $\widetilde{\times}: [0,B]^2 \to [0,B^2]$ computed by a ReLU network with $\O\left( \log^2\left(\eps^{-1}B\right)\right)$ parameters such that
    \begin{equation}
        \notag
        \sup_{x,y\in [0,B]}\left| \widetilde{\times}(x,y) - x y\right| \leq \eps,
    \end{equation}
    and $\widetilde{\times}(x,y) = 0$ if $x y=0$.
    \item ~\cite[Lemma 3.5]{telgarsky2017neural} There exists a function $R: [a,b] \to \R_{+}$ computed by a ReLU network with $\O\left(\log^4\left( a^{-1}b\right)\log^3(\eps^{-1}b)\right)$ parameters such that $\sup_{[a,b]}\left| R(x) - \frac{1}{x}\right| \leq \eps$.
\end{enumerate}
\end{lemma}

The following lemma establishes uniform approximation of polynomials and is a slight generalization of ~\cite[Lemma 3.4]{telgarsky2017neural}.

\begin{lemma} 
\label{poly_approx}
Let $\eps \in (0,1)$. Suppose that $P(x) = \sum_{k=1}^s \alpha_k \prod_{i=1}^{r_k} \left(x_{k,i}- a_{k,i}\right)$ is a polynomial with $\max_k r_k = r$ and $\alpha_k, a_{k,i} \in [0,1], \forall 1\leq k\leq s, 1\leq i\leq r_k$, and $P(x) \in [-1,+1]$ for $\forall x \in [0,1]^d$. Then there exists a function $N(x)$ computed by a ReLU network with $\O\left(s r \log\left( \eps^{-1}s r\right)\right)$ parameters such that $\sup_{[0,1]^d}\left| P(x) - N(x) \right| \leq \eps$.
\end{lemma}

\begin{proof}
It suffices to show that each monomial $P_k(x) = \prod_{i=1}^{r_k} \left( x_{k,i}-a_{k,i}\right)$ can be $\eps$-approximated using $\O\left(r \log\left( \eps^{-1} r\right) \right)$ parameters. Firstly, we need at most $r_k \leq r$ parameters to obtain $x_{k,i}-a_{k,i}, 1\leq i\leq r_k$ from a linear transformation. We can then apply Lemma \ref{simp_func} to perform successive multiplication. Note that we still have $\left| x_{k,i}-a_{k,i}\right| \leq 1$, which can be used to control the cumulative error of $\widetilde{\times}$.
\end{proof}

We are now ready to prove Theorem \ref{robust_training_error}. For convenience, we restate this theorem below.

\begin{theorem}
\label{robust_training_error_appendix}
Suppose that $\mathcal{D} \subset \B_{p}(0,1)$  with $p \in \{2,+\infty\}$ consists of $N$ data, 
and the two classes in $\mathcal{D}$ are $2\epsilon$-separated (cf. Definition \ref{separated_data}), where $\epsilon \in \left(0,\frac{1}{2}\right)$ is a constant. Let robustness radius $\delta < \frac{1}{2}\epsilon$, then there exists a classifier $f$ represented by a ReLU network with at most
\begin{equation}
    \notag
    \O\left( N d \log\left( \delta^{-1} d\right) + N\cdot \mathrm{polylog}(\delta^{-1}N)\right)
\end{equation}
parameters, such that $\hat{\L}_{\mathcal{D}}^{p,\delta}(f) = 0$.
\end{theorem}

\begin{proof}
\textit{(1). The case } $p=2$. First, we choose $C,\eps_1,\eps_2 > 0$ and $m \in\mathbb{Z}_{+}$ that satisfy
\begin{equation}
    \label{dist_ineq}
    C\left( (\delta^2+\eps_1)^m + \eps_2\right) \leq \frac{1}{4} < 4N \leq C\left( (R^2-\eps_1)^m - \eps_2\right).
\end{equation}
These constants will be specified later. Since for $\forall \bm{x}_0\in [0,1]^d$, $\bm{x} \to \|\bm{x}-\bm{x}_0\|^2$ is a polynomial that consists of $d$ monomials and with degree $2$, satisfying the conditions in Lemma \ref{poly_approx}, there exists a function $\phi_1$ computed by a ReLU network with $\O\left( d\log\left(\eps_1^{-1}d\right)\right)$ parameters such that $\sup_{\bm{x}\in[0,1]^d}\left| \phi_1(\bm{x}) - \|\bm{x}- \bm{x}_0\|^2\right| \leq \eps_1$. We may further assume that $\phi\left([0,1]^d\right) \subset [0,1]$, or otherwise we can consider $\sigma\left( \phi_1(\bm{x})\right) - \sigma\left(\phi_1(\bm{x})-1\right)$ instead.

Applying Lemma \ref{poly_approx} again, we can see that the function $x \to x^m$ on $[0,1]$ can be approximated with error $\eps_2$ by a function $\phi_2$ computed by a ReLU network with $\O\left( m\log\left(\eps^{-1}m\right)\right)$ parameters. Now we can see that $1+ C\cdot \phi_2 \circ \phi_1$ is computable by a ReLU network and takes value in $\left[1,\frac{5}{4}\right]$ when $\bm{x} \in \B(\bm{x}_0,\delta)$ and in $(4N+1,C+1)$ when $\bm{x}\notin \B(\bm{x}_0,R)$ (since $R \leq 1$).

The final step is to choose $\phi_3$ computed by a ReLU network with $\O\left( \log^4 C \log^3\left(NC\right)\right)$ parameters such that it approximates $\frac{1}{x}$ on $[1,C+1]$ with error $< \frac{1}{4N}$. Hence $\phi_3\circ\left( 1+ C\cdot \phi_2 \circ \phi_1 \right)$ is larger than $\frac{3}{4}$ inside $\B(\bm{x}_0,\delta)$ and smaller than $\frac{1}{2N}$ outside $\B(\bm{x}_0,R)$. This construction uses a total of $\O(W)$ parameters, where 
\begin{equation}
    \label{number_of_parameter}
    W = d \log\left( \eps_1^{-1} d\right) + m \log\left(\eps_2^{-1}m\right) + \log^4 C\log^3(NC).
\end{equation}
Finally, we choose
\begin{equation}
    \notag
    \eps_1 = \frac{R\delta(R-\delta)}{R+\delta}, \quad m = \max\left\{ 1, \log\frac{32N\delta}{R}\right\},\quad \eps_2 = \frac{1}{33N}\left( \frac{R(R^2+\delta^2)}{R+\delta}\right)^m,
\end{equation}
and $C = \frac{4N}{(R^2-\eps_1)^m - \eps_2} = \O\left( N\delta^{-2m}\right)$, which satisfies  \eqref{dist_ineq}. Plugging all expressions into \eqref{number_of_parameter}, we can see that
\begin{align*} 
    W = \O\left( d\left( \log d + \log \delta^{-1} + \log (R-\delta)^{-1}\right) + \log^7\left(\delta^{-1}N\right)\right).
\end{align*}
We denote this construction by $\psi(\bm{x};\bm{x}_0,\bm{\theta})$, where $\theta$ consists of all parameters. The arguments above show that there exists $\bm{\theta} = \bm{\theta(x_0)}$ such that $\psi(\bm{x;x_0,\theta}) > \frac{3}{4}$ when $\bm{x} \in \B(\bm{x}_0,\delta)$ and $\psi(\bm{x};\bm{x}_0,\bm{\theta}) < \frac{1}{2N}$ when $\bm{x} \notin \B(\bm{x}_0,R)$. Consider the function $\Psi(\bm{x};\bm{\theta}_{1:N}) =  4\sum_{i=1}^N \psi(\bm{x};\bm{x}_i,\bm{\theta}_i)-\frac{5}{2}$. The total number of parameters in $\Psi$ is $\Olog\left( Nd\right)$. Moreover, if we choose $\bm{\theta}_i = \bm{\theta}(\bm{x}_i)$ when $y_i = 1$ and $\bm{\theta}_i=0$ when $y_i=-1$, then $\Psi$ satisfies the condition in Theorem \ref{robust_training_error_appendix}.
\\

\textit{(2). The case } $p=\infty$. To obtain the same result under the $\ell_{\infty}$ norm, it suffices to construct a neural network with size $\O(d)$ parameters to represent the function $x \to \left\|x-x_0\right\|_{\infty}$; the remaining steps are exactly the same with the $\ell_2$ case.

Let $x^{(i)}$ denote the $i$-th coordinate of $\bm{x}$, then $\left\|\bm{x}-\bm{x}_0\right\|_{\infty} = \max_{1\leq i\leq d} \left| x^{(i)}-x_0^{(i)}\right|$. Since
\begin{equation}
    \notag
    \left| a \right| = \frac{1}{2}\left( \max\{ a,0\} + \max\{-a,0\}\right),
\end{equation}
we can see that $x^{(i)} \to \left| x^{(i)}-x_0^{(i)}\right|$ can be represented by a constant-size ReLU network. Moreover, the function $\max\{a,b\} = \frac{1}{2}\left(\left| a+b\right| + \left|a-b\right|\right)$, so that the function $(a_1,a_2,\cdots,a_d) \to \max_{1\leq i\leq d} a_d$ can be represented with $\O(d)$ parameters. To summarize, $\bm{x} \to \left\|\bm{x}-\bm{x}_0\right\|_{\infty}$ can be represented using a ReLU network of size $\O(d)$, as desired. 
\end{proof}

In the following, we prove Theorem \ref{robust_training_lower_bound}. 

\begin{theorem}[Restatement of Theorem \ref{robust_training_lower_bound}]
\label{robust_training_lower_bound_appendix}
Let $p \in \{2,+\infty\}$ and $\mathcal{F}_n$ be the set of functions represented by some ReLU network with at most $n$ parameters. If for any $2\epsilon$-separated data set $\mathcal{D}$ under $\ell_p$ norm, there exists a classifier $f \in \mathcal{F}_n$ such that $\hat{\L}_{\mathcal{D}}^{p,\delta}(f) = 0$, then it must hold that $n = \Omega(\sqrt{N d})$.
\end{theorem}

\begin{proof}
It follows from the assumption that given any data points $\bm{x}_1,\bm{x}_2,\cdots,\bm{x}_N$ which are pair-wise $2\epsilon$-separated,  there exists $f \in \mathcal{F}_n$ being able to achieve zero training error for any binary label. It directly follows from ~\cite[Theorem 6.1]{gao2019convergence} that
\begin{equation}
    \notag
    \text{VC-dim}(\mathcal{F}_n) = \Omega(Nd).
\end{equation}
On the other hand, suppose that $L \neq n$ is the depth of the neural network, then we have 
\begin{equation}
    \notag
    \text{VC-dim}(\mathcal{F}_n) = \O(nL\log(nL)) = \O(n^2).
\end{equation}
As a result, it follows that $n = \Tilde{\Omega}(\sqrt{Nd})$, as desired.
\end{proof}

\section{Proofs for Section \ref{general}}

\subsection{Proof of Theorem \ref{general_upper_bound}} \label{appendix:pf:general:ub}

The proof idea of Theorem~\ref{general_upper_bound} has two key steps. First, we construct a Lipschitz classifier $f^{*}$ based on distance function between a point and a close set that can $\epsilon-$ robustly classify $A,B$. Then we regard $f^{*}$ as the target function and use a ReLU network to approximate it to derive the $c\epsilon-$robust classifier. Before proving the theorem, we first introduce the two following useful conclusions, which also corresponding to the two steps of proof.

\begin{proposition}
\label{prop:dist_properties}
    For the separable $A,B \subset [0,1]^{d}$, we define $f^{*}(\bm{x}):=\frac{d_{\infty}(\bm{x},B)-d_{\infty}(\bm{x},A)}{d_{\infty}(\bm{x},A)+d_{\infty}(\bm{x},B)}$, which has the following properties:
    \begin{enumerate}
        \item 
        $f^{*}(\bm{x})$ can classify $A,B$ correctly i.e. $f^{*}(\bm{x})=\left\{\begin{array}{ll}
        1 ,& x \in A\\
        -1 ,& x \in B
        \end{array} \right.$. 
        \item
        $f^{*}(\bm{x})$ is a $\epsilon$-robust classifier i.e. for any perturbed input $x'$ that satisfies $||\bm{x}'-\bm{x}||_{\infty} \leq \epsilon$ can also be classified correctly.
        \item
        $f^{*}(\bm{x})$ is $\frac{1}{\epsilon}$-Lipschitz w.r.t. $\ell_{\infty}$ norm.
    \end{enumerate}
\end{proposition}
    
We can check these properties by the continuity and $1$-Lipschitz property of distance function $d_{\infty}(p,S)$.

\begin{lemma} \label{lem:approx:lip}
For any $L-$lipschitz function $f$ in $[0,1]^{d}$, there exists a function $\Tilde{f}$ implemented by ReLU network with at most $c_1(c_2 \epsilon/L)^{-d}(d^{2}+d \log d+d\log(1/\epsilon))$ parameters that satisfies $|f(\bm{x})-\Tilde{f}(\bm{x})|\leq \epsilon$ for any $\bm{x} \in [0,1]^{d}$, where $c_1$ and $c_2$ are constants.
\end{lemma}

This lemma provides a useful approximation tool for us, which is an improved version of Theorem 1 in \citet{yarotsky2017error}. Compared with Theorem 1 in \citet{yarotsky2017error}, we use Lipschitz property of function instead of high-order differetiability and focus on not the bound order when $\epsilon$ goes to zero but also more accurate bound order depending on $\epsilon, L$ and $d$. By a refined analysis of total approximation error, we can derive this lemma.

\begin{proof}[Proof of Theorem \ref{general_upper_bound}] By Lemma~\ref{lem:approx:lip}, we can approximate $f^{*}$ in Lemma~\ref{prop:dist_properties} satisfying uniform error at most $1-c$ via a ReLU network $f$ with at most $c_1(c_2(1-c)\epsilon)^{-d}(d^{2}+d \log d+d\log(1/(1-c)))$ parameters. Then, we prove the theorem by contradiction. Assume that there exists some perturbed input $\bm{x}'$ that is mis-classified and the original input $x$ is in $A$. So we know $f(\bm{x}') < 0$ and $f^{*}(\bm{x}) < \epsilon '$. This impiles
    $
        d_{\infty}(\bm{x}',A) < d_{\infty}(\bm{x}',B) < \frac{1+\epsilon '}{1-\epsilon '}d_{\infty}(\bm{x}',A) \nonumber
    $
    Combined with $d_{\infty}(\bm{x}',A)+d_{\infty}(\bm{x}',B) \geq d_{\infty}(A,B) \geq 2\epsilon$, we have $d_{\infty}(\bm{x}',A) > (1-\epsilon ')\epsilon = c \epsilon$, which is the contradiction.
\end{proof}

\subsection{Proof of Theorem \ref{lower_bound_general}} \label{appendix:pf:general:lb}

The main idea of proof is to estimate the lower bound of the family's VC-dimension via the definition of $c\epsilon$-robust family.

\begin{proof}[Proof of Theorem \ref{lower_bound_general}] 
    The key idea is to find some discrete points that can be shattered by the function family $\mathcal{F}_n$. 
    
    \textit{(1). The $p=\infty$ case. } We use $K$ to denote $\lfloor \frac{1}{2\epsilon} \rfloor + 1$, and we can divide $[0,1]^{d}$ into $(K-1)^{d}$ non-overlapping sub-cubes. Let $S$ be the set of all the vertices of sub-cubes, which has $K^{d}$ elements and can be represented by
    \begin{equation}
        S = \{\bm{x}_1,\bm{x}_2,\cdots,\bm{x}_{K^{d}}\} = \{(2\epsilon i_1,2\epsilon i_2, \cdots, 2\epsilon i_d)|0 \leq i_1,i_2,\cdots,i_d < K\}. \nonumber
    \end{equation}
    For any partition $I,J$ of $[K^{d}]$ ($I \cap J = \Phi, I \cup J = [K^{d}]$), let $A=\{\bm{x}_i|i \in I\}$ and $B=\{\bm{x}_j|j \in J\}$ be the positive and negative classes. Then we have $d_{\infty}(A,B) \geq 2\epsilon$. By the definition of $c\epsilon$-robust classifier family, there exists a classifier $f \in F$ classify $A,B$ correctly. Thus, the family $F$ shatter the subset $S \subset [0,1]^{d}$. By using the conclusion of Lemma \ref{lem:vc_dim}, we have
    $
        K^d \leq \text{VC-dim}(F) = \O(WL\log(N)) = \O(W^{2}\log(W)) 
    $
    where $L$ is the depth of networks and $W$ is the total number of parameters. 
    \\
    
    \textit{(2). The $p=2$ case. }Similar to the case of $p = \infty$, we need to construct a set $S \subset \B_2(0,1)$ such that the $\ell_2$-distance between any two points in $S$ is as least $2\epsilon$. 
    
    Specifically, we choose $S$ to be a $2\epsilon$-packing of $\B_2(0,1)$ with maximal size. Then we have that $\left| S\right| \geq \mathcal{P}(\B_2(0,1),2\epsilon) \geq \mathcal{C}(\B_2(0,1),2\epsilon) \geq (2\epsilon)^{-d}$, by Propositions \ref{covering_lower_bound} and \ref{prop:packing_covering}. Similar to the $p = \infty$ case, robustness implies that $S$ can be shattered by $\mathcal{F}_n$, so that $K^d = \O(W^2\log W)$ and the conclusion follows.
\end{proof}

\section{Proofs for Section \ref{linear_separable_setting}}
In this section, we present the proof of Theorem \ref{linear_separable_lower_bound} and \ref{learn_linear_separable}.

\begin{theorem}[Restatement of Theorem \ref{linear_separable_lower_bound}]
\label{linear_separable_lower_bound_appendix}
Let $\epsilon \in (0,1)$ be a small constant, $p \in \{2,\infty\}$ and $\mathcal{F}_n$ be the set of functions represented by ReLU networks with at most $n$ parameters. There exists a sequence $N_d = \Omega\left((2\epsilon)^{-\frac{d-1}{6}}\right), d \geq 1$ and a universal constant $C > 0$ such that the following holds: for any $c \in (0,1)$, there exists two linear separable sets $A,B \subset [0,1]^{d}$ that are $2\epsilon$-separated under $\ell_p$-norm, such that for any $\mu_{0}-$balanced distribution $P$ 
on the supporting set $S = A \cup B$ and robust radius $c\epsilon$ we have
\begin{equation}
     \inf\left\{ \L_{P}^{p,c\epsilon}(f) : f \in \mathcal{F}_{N_d}\right\} \geq C \mu_0.
    \nonumber
\end{equation}
\end{theorem}

\begin{proof}
\textit{(1). The $p=\infty$ case. } Define
\begin{equation}
    \notag
    S_{\phi} = \left\{ \left( \frac{i_1}{K}, \frac{i_2}{K}, \cdots, \frac{i_{d-1}}{K}, \frac{1}{2}+ c\eps\cdot\phi(i_1,i_2,\cdots,i_{d-1}) \right) : 1 \leq i_1,i_2,\cdots, i_{d-1} \leq K \right\},
\end{equation}
and
\begin{equation}
    \notag
    \widetilde{S} = \left\{ \left( \frac{i_1}{K}, \frac{i_2}{K}, \cdots, \frac{i_{d-1}}{K}, \frac{1}{2} \right) : 1 \leq i_1,i_2,\cdots, i_{d-1} \leq K \right\},
\end{equation}
where $K = \floor{\frac{1}{2\eps}}$, and $\phi: \{1,2,\cdots,K\}^{d-1} \to \left\{ -1,+1\right\}$ be an arbitrary mapping. For a vector $\bm{x} \in \R^d$, we use $x^{(i)}$ to denote its $i$-th component. Let $A_{\phi} = S_{\phi} \cap \left\{ \bm{x}\in\R^d: x^{(d)} > \frac{1}{2}\right\}$, $B_{\phi} = S_{\phi} - A_{\phi}$ and $\mu$ be the uniform distribution on $S$. It's easy to see that $A$ and $B$ are linear separable by the hyperplane $x^{(d)} = \frac{1}{2}$. Moreover, we clearly have $d(A,B) \geq 2\eps$. We will show that there exists some choice of $\phi$ such that robust classification of $A_{\phi}$ and $B_{\phi}$ with $(c\eps,1-\alpha)$-accuracy requires at least $\Omega\left( K^{(d-1)/6}\right)$ parameters.

Assume that for any choices of $\phi$, the induced sets $A_{\phi}$ and $B_{\phi}$ can always be robustly classified with $(c\eps,1-\alpha)$-accuracy by a ReLU network with at most $M$ parameters. Then, we can construct an \textit{enveloping network} $F_{\bm{\theta}}$ with $M-1$ hidden layers, $M$ neurons per layer and at most $M^3$ parameters such that any network with size $\leq M$ can be embedded into this envelope network. As a result, $F_{\bm{\theta}}$ is capable of $(c\eps,1-\alpha)$-robustly classify any sets $A_{\phi},B_{\phi}$ induced by arbitrary choices of $\phi$. We use $R_{\phi}$ to denote the subset of $S_{\phi} = A_{\phi} \cup B_{\phi}$ satisfying $\left| R_{\phi}\right| = (1-\alpha) \left| S_{\phi}\right| = (1-\alpha) K^{d-1}$ such that $R_{\phi}$ can be $c\eps$-robustly classified.

Consider the projection operator $\mathcal{P}$ onto the hyperplane $x^{(d)} = \frac{1}{2}$. For any set $C \in \R^d$, we use $\widetilde{C}$ to denote $\mathcal{P}(C)$. Then $c\eps$-robustness implies that the labelled data set
\begin{equation}
    \notag
    R_{\phi}^{+} = \left\{ (\bm{x},y): \bm{x} \in \widetilde{R}_{\phi}, y = \phi\left( K x^{(1)}, \cdots, K x^{(d-1)}\right)\right\} 
\end{equation}
can be correctly classified by $F_{\bm{\theta}}$, with appropriate choices of parameters.

Let $V = \frac{1}{2} K^{d-1}$ and $\hat{\mathcal{R}}_{\phi}$ be the collection of all labelled $V$-subset (i.e. subset of size $V$) of $R_{\phi}^{+}$. For each $V$-subset $R$ of $\widetilde{S}$, we use $\mathcal{G}(R)$ to denote the set of all labelings of $R$, so that $\left|\mathcal{G}(R)\right| = 2^V$.

Note that for each labelled $V$-subset $T$, there exists at most $2^{K^{d-1}-V}$ different choices of $\phi$ such that $T \subset R_{\phi}^+ $ (or, equivalently, $T \in \hat{\mathcal{R}}_{\phi}$): this is because the value of $\phi$ on data points in $T$ has been specified by their labels, and there are two choices for each of the remaining $K^{d-1}-V$ points in $\{1,2,\cdots,K\}^{d-1}$. As a result, we have
\begin{equation}
\notag
    \left| \cup_{\phi} \hat{\mathcal{R}}_{\phi} \right| \geq 2^{-(K^{d-1}-V)} \sum_{\phi} \left|\hat{\mathcal{R}}_{\phi} \right| = 2^V \binom{(1-\alpha)K^{d-1}}{V}.
\end{equation}
On the other hand, the total number of $V$-subset of $\widetilde{S}$ is $\binom{K^{d-1}}{V}$, thus there must exists a $V$-subset $\mathcal{V}_0 \subset \widetilde{S}$, such that at least
\begin{equation}
    \binom{K^{d-1}}{V}^{-1}\cdot 2^V \binom{(1-\alpha)K^{d-1}}{V} \geq \left( \frac{2\left( (1-\alpha)k^{d-1}-V\right)}{K^{d-1}-V}\right)^V
\end{equation}
different labelings of $\mathcal{V}_0$ are included in $\cup_{\phi} \hat{\mathcal{R}}_{\phi}$. Since $F_{\bm{\theta}}$ can correctly classify all elements (which are $V$-subsets) in $\cup_{\phi} \hat{\mathcal{R}}_{\phi}$, it can in particular classify the set $\mathcal{V}_0$ with at least $\left( \frac{2\left( (1-\alpha)k^{d-1}-V\right)}{K^{d-1}-V}\right)^V$ different assignments of labels. Let $d_{VC}$ be the VC-dimension of $F_{\bm{\theta}}$, then by Lemma \ref{growth_func_upper_bound}, either $d_{VC} \geq V = \frac{1}{2}K^{d-1}$, or
\begin{equation}
    \notag
    \left( 2(1-2\alpha)\right)^V \leq  \left( \frac{2\left( (1-\alpha)K^{d-1}-V\right)}{K^{d-1}-V}\right)^V \leq \Pi_{F_{\bm{\theta}}}(V) \leq \left( \frac{eV}{d_{VC}}\right)^{d_{VC}},
\end{equation}
where $\Pi$ is the growth function. The RHS is increasing in $d_{VC}$ as long as $d_{VC} \leq V$. When $\alpha \leq \frac{1}{10}$, we have $2(1-2\alpha) > (10e)^{1/10}$, so that $d_{VC} \geq \frac{1}{10}V = \frac{1}{20}K^{d-1}$. Finally, since $F_{\bm{\theta}}$ has at most $M^3$ parameters, classical bounds on VC-dimension \citep{bartlett2019nearly} imply that $M = \Omega\left( K^{(d-1)/6}\right)$, as desired.
\\

\textit{(2). The $p=2$ case. } Let $P$ be an $2\epsilon$-packing of the unit ball $\B_{d-1}$ in $\R^{d-1}$. Since the packing number $\mathcal{P}(\B_{d-1},\|\cdot\|,2\epsilon) \geq \mathcal{C}(\B_{d-1},\|\cdot\|_2,2\epsilon) \geq (2\epsilon)^{-(d-1)}$ by Propositions \ref{covering_lower_bound} and \ref{prop:packing_covering}, where $\mathcal{C}(\Theta,\|\cdot\|,\epsilon)$ is the $\epsilon$-covering number of a set $\Theta$. For any $\lambda \in (0,1)$, we can consider the construction
\begin{equation}
    \notag
    S_{\phi} = \left\{ \left( \bm{x}, \frac{1}{2} + \epsilon_0 \cdot \phi(\bm{x})\right): \bm{x} \in P\right\},
\end{equation}
where $\phi: P \to \{-1,+1\}$ is an arbitrary mapping. It's easy to see that all points in $S_{\phi}$ with first $d-1$ components satisfying $\left\|\bm{x}\right\|_2 \leq \sqrt{1-\eps_0^2}$ are in the unit ball $\B_d$, so that by choosing $\eps_0$ sufficiently small, we can guarantee that $\left|S_{\phi} \cap \B_{d}\right| \geq \frac{1}{2}(2\epsilon)^{-(d-1)}$. For convenience we just replace $S_{\phi}$ with $S_{\phi}\cap \B_d$ from now on.

Let $A_{\phi} = S_{\phi} \cap \left\{ \bm{x}\in\R^d: x^{(d)} > \frac{1}{2}\right\}$, $B_{\phi} = S_{\phi} - A_{\phi}$. It's easy to see that for arbitrary $\phi$, the construction is linear-separable and satisfies $2\epsilon$-separability. The remaining steps are just identical to the $\ell_{\infty}$ case.
\end{proof}

\begin{theorem}[Restatement of Theorem \ref{learn_linear_separable}]
\label{learn_linear_separable_appendix}
For any two linear-separable $A,B \subset [0,1]^{d}$, a distribution $P$ on the supporting set $S= A \cup B$, $\delta > 0$ and $\beta > 0$, let $H$ be the family of $d-$dimensional hyperplane classifiers. Then, there exists a poly-time efficient algorithm $\mathcal{A}: 2^{S} \rightarrow H$, for $N= \Omega(d/\beta^{2})$ training instances independently randomly sampled from $P$, with probability $1-\delta$ over samples, we can use the algorithm $\mathcal{A}$ to learn a classifier $\hat{f} \in F$ such that
\begin{equation}
    \notag
    \L_{P}(\hat{f}) \leq \beta,
\end{equation}
where $\L_{P}(f):= \mathbb{P}_{(x,y) \sim P}\{y \ne f(x)\}$ denotes the standard test error.
\end{theorem}

\begin{proof}
\label{proof_linear_upper}
We i.i.d. sample $N$ instances from the data distribution $P$, and use $T$ to denote the training dataset. By Lemma \ref{lem:vc_bound}, with probabililty at least $1-\delta$, we have
\begin{equation}
    \L_P(h) \leq \L_T(h) + \mathcal{O}\left(\sqrt{\frac{d}{N}}\right) , \forall h \in H \nonumber
\end{equation}
By conclusions of \citet{boser1992training} and results of convex optimization, we have a poly-time algorithm $\mathcal{A}:2^{S}\rightarrow H$ such that $\L_T(\mathcal{A}(T)) \leq \frac{\beta}{2}$, and we use $\hat{f}$ to denote $\mathcal{A}(T)$. Finally, when $N = \Omega(d/{\beta^{2}})$ is sufficient large, we have $\L_P(\hat{f}) \leq \frac{\beta}{2} + \frac{\beta}{2} = \beta$.
\end{proof}

\section{Proofs for Section \ref{low_dimensional_manifold_setting}}

\subsection{Proof of Lemma \ref{manifold_appr_lemma}}

\begin{theorem}[Restatement of Lemma \ref{manifold_appr_lemma}]
\label{manifold_appr_lemma_appendix}
Let $\mathcal{M} \subset [0,1]^{d}$ be a $k-$dimensional compact poly-partitionable Riemannian manifold with the condition number $\tau > 0$. For any small $\delta > 0$ and a $L-$lipschitz function $f:\mathcal{M} \rightarrow \mathbb{R}$, there exists a function $\hat{f}$ implemented by ReLU network with at most 
\begin{equation}
    \Tilde{\mathcal{O}}\left((\sqrt{d}L/\delta)^{-\Tilde{k}}\right) \nonumber
\end{equation}
parameters, such that $|f-\hat{f}| < \delta$ for any $x \in \mathcal{M}$, where $\Tilde{k}$ is the same as Theorem \ref{manifold_upper_bound}.
\end{theorem}

\begin{proof}
\label{proof_manifold_ub}

The full proof has six steps, and we finally construct a ReLU network as the following form
\begin{equation}
    \hat{f} = \sum_{i=1}^{N} (\hat{f}_{i} \circ \phi_{i})
    \hat{\times} (\hat{\rho}_{i} \circ T_{i}) \hat{\times}
    (\hat{I}_{\theta} \circ \hat{d_{i}^{2}}), \nonumber
\end{equation}
where all $\hat{f}_{i},\phi_{i},\hat{\rho}_{i},T_{i},\hat{I}_{\Delta},\hat{d_{i}^{2}}$ and $\hat{\times}$ are implemented by sub ReLU networks, and these notation's detail will be introduced step by step. 

As the above form shows, three sub-network groups will be combined by multiplication approximator $\hat{\times}$, where each group corresponds to a factor in the partition-of-unity-based decomposition of $f$ (i.e. $f=\sum_{i=1}^{N} f \times \rho_{i} \times I\{x \in U_{i}\}$, and where $\{\rho_i\}_{i \in [N]}$ satisfying $\sum_{i \in [N]} \rho_{i} = 1$ is a partition of unity on an atlas $\{U_{i}\}_{i \in [N]}$).

{\bf Step 1: Construct poly-partition of unity on $\mathcal{M}$}

Consider a open cover $\{B_{r}(x)\}_{x \in \mathcal{M}}$ on $\mathcal{M}$, where we use $B_{r}(c)$ to denote the Euclidean neighborhood with center $c$ and radius $r$. Due to the compactness of manifold $\mathcal{M}$, we know there exists a finite open cover $\{B_{r}(x_i)\}_{i \in I}$ indexed by a finite sub-index set $I$, which satisfies $\mathcal{M} \subset \bigcup_{i \in I}B_{r}(x_i)$.

Then, we estimate the cardinal number of index set. By the conclusions of \cite{niyogi2008finding}, when we select radius $r$ satisfying $r < \tau / 2$, we have the following lemma, which gives an lower bound of $k-$dimensional volume of the local neighborhood of $\mathcal{M}$.

\begin{lemma}
~\cite[Lemma 5.3]{niyogi2008finding}
\label{lem:manifold_vol}
Let $c \in \mathcal{M}$. Now consider $U=\mathcal{M} \cap B_{r}(c)$. Then $\operatorname{vol}(U) \geq(\cos (\theta))^{k}$ $\operatorname{vol}\left(B_{r}^{k}(c)\right)$ where $B_{r}^{k}(c)$ is the $k$-dimensional ball in $T_{c}$ centered at $c, \theta=$ $\arcsin (r / 2 \tau)$. All volumes are $k$-dimensional volumes where $k$ is the dimension of $\mathcal{M}$.
\end{lemma}

Recall the relation between the covering number $\mathcal{N}(\mathcal{M},d_2,r)$ and the packing number $\mathcal{P}(\mathcal{M},d_2,r)$, and then we have
\begin{equation}
\begin{aligned}
\mathcal{N}(\mathcal{M},d_2,r) & \leq \mathcal{P}(\mathcal{M},d_2,r/2)
\\ & \leq \frac{\operatorname{vol}(\mathcal{M})}{(\cos (\theta))^{k}\operatorname{vol}\left(B_{r/2}^{k}(c)\right)}
\\ & \leq c_N \frac{\operatorname{vol}(\mathcal{M})}{r^{k}}, \nonumber
\end{aligned}
\end{equation}
where $c_N$ is a constant that only exponentially depends on $k \log k$.

By the poly-partitionable and smoothness properties of Riemannian manifold $\mathcal{M}$, there exists a collection $\{U_i,T_i,\rho_i\}_{i \in \mathcal{N}(\mathcal{M},d_2,r)}$ such that $\{U_i,T_i\}$ compose a tangent-space-induced atlas and $\{\rho_i\}$ also compose a poly-partition of unity on $\mathcal{M}$. So we can decompose $f$ as $f=\sum_{i=1}^{N}f\rho_i$, where we use notation $N$ to denote $\mathcal{N}(\mathcal{M},d_2,r)$ so as to simplify the written process.

{\bf Step 2: Local almost isotropic transformation via random projection}

To achieve dimensional reduction of $f$ in the local neighborhood $U_i$, we will use the following random projection technique proposed by \cite{baraniuk2009random}.

\begin{lemma}
~\cite[Theorem 3.1]{baraniuk2009random}
\label{lem:random_proj}

Let $\mathcal{M}$ be a compact $k$-dimensional sub-manifold of $\mathbb{R}^{d}$ having condition number $1 / \tau$. Fix $0<\delta<1$ and $0<\eta<1$. Let $A$ be a random orthoprojector from $\mathbb{R}^{d}$ to $\mathbb{R}^{\Tilde{k}}$ with
\begin{equation}
    \Tilde{k}=O\left(\frac{k \log \left( d \operatorname{vol}(\mathcal{M})  \tau^{-1} \delta^{-1}\right) \log (1 / \eta)}{\delta^{2}}\right). \nonumber
\end{equation}
If $\Tilde{k} \leq d$, then with probability at least $1-\eta$ the following statement holds: For every distinct pair of points $x, y \in \mathcal{M}$,
\begin{equation}
    (1-\delta) \sqrt{\frac{\Tilde{k}}{d}} \leq \frac{\|A x-A y\|_{2}}{\|x-y\|_{2}} \leq(1+\delta) \sqrt{\frac{\Tilde{k}}{d}} . \nonumber
\end{equation}
\end{lemma}

Since we can select $\eta$ is very close to $1$ in order to the probability $1-\eta>0$, there exists a orthoprojector $A_i$ for sub-manifold $U_i$ by applying Lemma \ref{lem:random_proj}. And we use $V_r$ to denote the uniform upper bound of $\operatorname{vol}(U_i)$, which makes the uniform dimension $\Tilde{k} = O\left(k \log d \right)$ for each $i \in [N]$. Let local almost isotropic transformation $\phi_i(x) = \frac{1}{2}A_i(x-c_i) + \frac{1}{2}\mathbb{1} $, where we use $\mathbb{1}$ to denote the vector $(1,1,...1) \in \mathbb{R}^{\Tilde{k}}$, and then we know $\phi_i(U_i) \subset [0,1]^{\Tilde{k}}$. 

{\bf Step 3: Approximate Lipschitz mapping $f \circ \phi_{i}^{-1}$ by $\hat{f}_{i}$}

To approximate $f \circ \phi^{-1}:[0,1]^{\Tilde{k}} \rightarrow \mathbb{R}$ via ReLU networks, we first caculate the Lipschitzness of it. For any pair $x,y$ of $\phi_i(U_i)$, we have
\begin{equation}
    \begin{aligned}
        | f \circ \phi^{-1}(x) - f \circ \phi^{-1}(y)|
        & \leq L\| \phi^{-1}(x)-\phi^{-1}(y) \|_{\infty}
        \\
        & \leq L\| \phi^{-1}(x)-\phi^{-1}(y) \|_{2}
        \\
        & \leq  \frac{2 L}{1-\delta}\sqrt{\frac{d}{\Tilde{k}}}
        \|x-y\|_{2}
        \\
        & \leq \frac{2 L \sqrt{d}}{1-\delta}
        \|x-y\|_{\infty}. \nonumber
    \end{aligned}
\end{equation}
The first inequality is due to the Lipschitzness of function $f$. The second and last equality is the equivalence between $\ell_2$ norm and $\ell_{\infty}$ norm. The third inequality uses the isotropic property of the orthoprojector $A_i$. So $f \circ \phi_i^{-1}$ is a $\frac{2 L \sqrt{d}}{1-\delta}$ Lipschitz mapping from $[0,1]^{\Tilde{k}}$ to $\mathbb{R}$.

By using Lemma \ref{lem:approx:lip}, there exists a ReLU network $\hat{f}_i$ with at most
\begin{equation}
    c_1 {\left( \frac{c_2 {\epsilon}_1 (1-\delta)}{2 L \sqrt{d}}\right)}^{-\Tilde{k}}(\Tilde{k}^{2}+\Tilde{k} \log \Tilde{k}+\Tilde{k} \log\frac{1}{\epsilon_1}) \nonumber
\end{equation}
parameters such that for any $x \in \phi_i(U_i)$, we have the uniform error $\epsilon_1$ as 
\begin{equation}
    |f \circ \phi_i^{-1}(x) - \hat{f}_i(x) | \leq \epsilon_1. \nonumber
\end{equation}
Notice that $\phi_i$ is a linear mapping so that we can use a ReLU network with only one layer to represent it, which shows that we can approximate $f$ efficiently in the local neighborhood $U_i$.

{\bf Step 4: Approximate simple piecewise polynomial $\rho_{i} \circ T_{i}^{-1}$ by $\hat{\rho}_{i}$}

According to the poly-partitionable property of manifold $M$ and Lemma \ref{poly_approx}, there exists a ReLU network $\hat{\rho}_i$ with at most $O(k \log (k / \epsilon_2))$ parameters such that for any $x \in T_i(U_i) \subset [0,1]^{k}$, we have the uniform error $\epsilon_2$ as
\begin{equation}
    |\rho_i \circ T_i^{-1}(x)-\hat{\rho}_i(x)| \leq \epsilon_2 , \nonumber
\end{equation}
where $T_i$ is composed by the tangent vectors of $c_i$ and is scaled and translated to ensure $T_i(U_i) \subset [0,1]^{k}$.

{\bf Step 5: Determine the corresponding neighborhood for input}

Notice that $\operatorname{supp}(\rho_i) \subset U_i$ but $\hat{\rho}_i$ may be non-zero for some point $[0,1]^{k}/T_i(U_i)$, so we need to determine the corresponding chart for input $x \in \mathcal{M}$ by ReLU networks. Inspired by \citet{chen2019efficient}, we construct indicate approximator $\hat{I}_{\theta}$ and $\ell_2$ distance approximators $\{ \hat{d_i^{2}} \}_{i \in [N]}$ based on quadratic approximator in Lemma \ref{simp_func} to approximate the neighborhood's indicator $I\{x \in U_i\}$, which relies upon the following identical equations
\begin{equation}
    I\{x \in U_i\} = I\{\|x-c_i\|_{2}^{2} < r^{2}\}
    =I\{(\cdot) < r^{2}\} \circ d_i^{2}(x), \nonumber
\end{equation}
where $d_i^{2}(x)$ denotes the square of $\ell_2$ distance between $x$ and $c_i$. Then, if $\hat{I}_{\theta} \approx I\{(\cdot) < r^{2}\}$ and $\hat{d_i^{2}} \approx d_i^{2}$, we have $\hat{I}_{\theta} \circ \hat{d_i^{2}} \approx I\{(\cdot) < r^{2}\} \circ d_i^{2} = I\{x \in U_i\}$, which determines the corresponding chart approximately.

Assume that the uniform error of square distance approximator is $\epsilon_q$ (i.e. $|d_i^{2}-\hat{d_i^{2}}| \leq \epsilon_q$ for any $x \in [0,1]^{d}$). In fact, functions computed by ReLU networks are piecewise linear but the indicator functions are not continuous, so we need to relax the indicator such that $\hat{I}_{\theta}(x) = 1$ for $x \leq r^{2}+\epsilon_q-\theta$, $\hat{I}_{\theta}(x) = 0$ for $x \geq r^{2}-\epsilon_q$ and $\hat{I}_{\theta}$ is linear in $(r^{2}+\epsilon_q-\theta,r^{2}-\epsilon_q)$.

To correct the difference between indicator and its approximator, we will bound the value of function $f$ such that the magnitude of $f(x)$ is sufficient small when $x$ is nearly on the boundary of $U_i$. Intuitively, for any $y \in \partial(U_i)$, we have
\begin{equation}
    f\rho_i(y) = 0 . \nonumber
\end{equation}

This is due to $\operatorname{supp}(\rho_i) \subset U_i$, which implies that we only need estimate the upper bound of $\|x-y\|_{2}$ for the Lipshcitzness of $f$ and smoothness of $\rho_i$, where $x$ is nearly on $\partial U_i$. Indeed, we can prove that for any $x U_i':=\in U_i/B_{\sqrt{r^2-\theta}}(c_i)$, there exists $y \in \partial U_i$ such that $\|x-y\|_{2}=O(\theta)$ ~\cite[Lemma 3]{chen2019efficient}.

{\bf Step 6: Estimate the total error}

We combine three sub-network groups as
\begin{equation}
    \hat{f} = \sum_{i=1}^{N} (\hat{f}_{i} \circ \phi_{i})
    \hat{\times} (\hat{\rho}_{i} \circ T_{i}) \hat{\times}
    (\hat{I}_{\theta} \circ \hat{d_{i}^{2}}). \nonumber
\end{equation}
Next, we estimate the  total error between $f$ and $\hat{f}$. For any $x \in M$, we use $g_i$ to denote $(\hat{f}_{i} \circ \phi_{i}) \hat{\times} (\hat{\rho}_{i} \circ T_{i})$, $I_i$ to denote $I\{x \in U_i\}$ and $\hat{I}_i$ to denote $\hat{I}_{\theta} \circ \hat{d_{i}^{2}}$, then we have
\begin{equation}
    \begin{aligned}
        |f(x)-\hat{f}(x)| &= 
        \left|\sum_{i=1}^{N} f\rho_i - \sum_{i=1}^{N} g_i \hat{\times} \hat{I}_i\right|
        \\
        & \leq
        \left|\sum_{i=1}^{N} f\rho_i-g_i I_i\right|
        + \left|\sum_{i=1}^{N} g_i\times I_i-g_i \hat{\times} \hat{I_i}\right|
        \\
        & = \left|\sum_{i:x \in U_i} f\rho_i-(\hat{f}_{i} \circ \phi_{i}) \hat{\times} (\hat{\rho}_{i} \circ T_{i})\right|
        + \left|\sum_{i=1}^{N} g_i\times I_i-g_i \hat{\times} \hat{I_i}\right| 
        \\
        & \leq \left|\sum_{i:x \in U_i} ((f\circ\phi_i^{-1}-\hat{f}_i)\circ\phi_i)\rho_i\right|+\left|\sum_{i:x \in U_i}(\hat{f}_{i} \circ \phi_{i})\times\rho_i-(\hat{f}_{i} \circ \phi_{i}) \hat{\times} (\hat{\rho}_{i} \circ T_{i})\right|
        \\
        & \quad
        +\left|\sum_{i=1}^{N} g_i\times I_i-g_i \hat{\times} \hat{I_i}\right|. \nonumber
        \\
    \end{aligned}
\end{equation}

The second identical equation is due to $\operatorname{supp}(\rho_i) \subset U_i$. Notice that $\sum_{i \in [N]}\rho_i=1$, then the first term satisfies that 
$$\left|\sum_{i:x \in U_i} ((f\circ\phi_i^{-1}-\hat{f}_i)\circ\phi_i)\rho_i\right| \leq \left(\sum_{i:x\in U_i} \rho_i\right)\max_{i:x\in U_i}\{|f\circ\phi_i^{-1}-\hat{f}_i|\}\leq \epsilon_1.$$

By the approximation of $\hat{\times}$, the second term satisfies that
$$
\left|\sum_{i:x \in U_i}(\hat{f}_{i} \circ \phi_{i})\times\rho_i-(\hat{f}_{i} \circ \phi_{i}) \hat{\times} (\hat{\rho}_{i} \circ T_{i})\right|
\lesssim \left|\sum_{i:x \in U_i}(\hat{f}_{i} \circ \phi_{i})\times((\rho_i\circ T_i^{-1}-\hat{\rho}_i)\circ T_i)\right| \leq c_f N \epsilon_2.
$$
where $c_f$ is the uniform upper bound the value of $\{\hat{f}_i\}_{i\in [N]}$. And the third term satisfies that
$$
\left|\sum_{i=1}^{N} g_i\times I_i-g_i \hat{\times} \hat{I_i}\right| \lesssim \left|\sum_{i=1}^{N} g_i\times (I_i- \hat{I_i})\right| \leq \sum_{i=1}^{N}\max_{x\in U_i'}|g_i| \lesssim \sum_{i=1}^{N}\max_{x\in U_i'}|f\rho_i|=O(N\theta).
$$
Finally, we choose $\epsilon_1=O(\epsilon)$ and $\epsilon_2=\theta=O(\epsilon/N)$ to control the total error boounded by $\epsilon$ and derive the upper bound for the size of network in Lemma \ref{manifold_appr_lemma_appendix}.

\end{proof}

\subsection{Proof of Theorem \ref{manifold_lower_bound}}

\begin{theorem}[Restatement of Theorem \ref{manifold_lower_bound}]
\label{manifold_lower_bound_appendix}
Let $\epsilon \in (0,1)$ be a small constant. There exists a sequence $\{N_k\}_{k \geq 1}$ that satisfies $N_k = \Omega\left((2\epsilon\sqrt{d/k})^{-\frac{k}{2}}\right)$.
and a universal constant $C_1 > 0$ such that the following holds: let $\mathcal{M} \subset [0,1]^{d}$ be a complete and compact $k-$dimensional Riemannian manifold with non-negative Ricci curvature , then there exists two $2\epsilon$-separated sets $A,B \subset \mathcal{M}$ under $\ell_{\infty}$ norm, such that for any $\mu_{0}-$balanced distribution $P$ on the supporting set $S = A \cup B$ and robust radius $c \in (0,1)$, we have
\begin{equation}
     \inf\left\{ \L_{P}^{\infty,c\epsilon}(f) : f \in F_{N_k}\right\} \geq C_1 \mu_0.
    \nonumber
\end{equation}
\end{theorem}

\begin{proof}
\label{proof_manifold_lb}

Our proof relies on the following propositions.

\begin{lemma}
~\cite[Proposition 6.3]{niyogi2008finding}
\label{lem:manifold_dist}

Let $\mathcal{M}$ be a sub-manifold of $\mathbb{R}^{d}$ with condition number $1 / \tau$. Let $p$ and $q$ be two points in $\mathcal{M}$ such that $\|x-y\|_{2}=r$. Then for all $r \leq \tau / 2$, the geodesic distance $d_{\mathcal{M}}(p, q)$ is bounded by
\begin{equation}
    d_{\mathcal{M}}(x, y) \leq \tau-\tau \sqrt{1-2r/\tau} . \nonumber
\end{equation}
\end{lemma}

By Lemma \ref{lem:manifold_dist}, we know that $d_{\mathcal{M}}(x, y) \leq \tau-\tau \sqrt{1-2r/\tau} \leq 2r$ when $r \leq \tau/2$.

\begin{lemma}
~\cite[Bishop-Gromov Volume Comparison Theorem]{bishop1964relation}
\label{lem:B-G-V-C-T}
Let $\mathcal{M}$ is a complete Riemannian manifold with Ricci curvature $Ric \geq (k-1) l$, and $p \in \mathcal{M}$ is an arbitrary point. Then the function
\begin{equation}
    r \mapsto \frac{\operatorname{vol}\left(B_{\mathcal{M},r}(p)\right)}{\operatorname{vol}\left(B_{r}^{l}\right)}   \nonumber
\end{equation}
is a non-increasing function which tends to 1 as r goes to 0 , where $B_{\mathcal{M},r}(p)$ is the $\mathcal{M}$'s geodesic ball of radius $r$ and center $p$, and $B_{r}^{l}$ is a geodesic ball of radius $r$ in the space form $\mathcal{M}_{l}^{k}$. In particular, $\operatorname{vol}\left(B_{\mathcal{M},r}(p)\right) \leq \operatorname{vol}\left(B_{r}^{l}\right)$.
\end{lemma}

By Lemma \ref{lem:B-G-V-C-T} and the non-negativeness of $\mathcal{M}$'s Ricci curvature, we know $\operatorname{vol}(B_{\mathcal{M},r}(c)) \leq \operatorname{vol}(B_{r}^{0})=r^{k}V_k$, where $V_k$ denotes the volume of the unit ball in $\mathbb{R}^{k}$. Recall the relation between the covering number $\mathcal{N}_{\mathcal{M}}(r)$ and the packing number $\mathcal{P}_{\mathcal{M}}(r)$ on the manifold $\mathcal{M}$, then we have
\begin{equation}
    \mathcal{P}_{\mathcal{M}}(r) \geq \mathcal{N}_{\mathcal{M}}(2 r) \geq \frac{\operatorname{vol}(\mathcal{M})}{(2 r)^{k}V_k} = \Omega\left(\frac{\operatorname{vol}(\mathcal{M})k^{\frac{k}{2}}}{r^{k}}\right) . \nonumber
\end{equation}
By choosing $r=2\epsilon\sqrt{d}$, we know that there are at least $\Omega\left((2\epsilon\sqrt{d/k})^{-k}\right)$ points on $\mathcal{M}$ such that the $\ell_{\infty}$ distance between each pair points of these is more than $2\epsilon$, where we use $\mathcal{Q}$ to denote the set of these selected points. The remain of proof is similar to the latter half of proof for Theorem \ref{linear_separable_lower_bound_appendix}.

Let $S = \mathcal{Q}$ be the supporting set. Assume that for any partition $A,B$ of $S$ such that $A \cup B = S$ and $A \cap B = \emptyset$, there exists a classifier $f \in F_{N_k}$ that robustly classifies $A$ and $B$ with at least $1-\alpha$ accuracy. Next, we estimate the lower and upper bounds for the cardinal number of the vector set 
\begin{equation}
    R := \{(f(x))_{x \in \mathcal{Q}}|f \in F_{N_k}\} . \nonumber
\end{equation}
Let $n$ denote $|\mathcal{Q}|$, then we have
\begin{equation}
    R=\{(f(x_1),f(x_2),...f(x_n))|f\in F_{N_k}\} , \nonumber
\end{equation}
where $\mathcal{Q}=\{x_1,x_2,...,x_n\}$.

On one hand, we know that for any $u \in \{-1,1\}^{n}$, there exists a $v \in R$ such that $d_H(u,v) \leq \alpha n$, where $d_H(\cdot,\cdot)$ denotes the Hamming distance, then we have
\begin{equation}
    |R| \geq \mathcal{N}(\{-1,1\}^{n},d_H,\alpha n) \geq \frac{2^{n}}{\sum_{i=0}^{\alpha n}\binom{n}{i}} . \nonumber
\end{equation}

On the other hand, by applying Lemma \ref{growth_func_upper_bound}, we have 
\begin{equation}
    \frac{2^{n}}{\sum_{i=1}^{\alpha n}\binom{n}{i}} \leq |R| \leq \Pi_{F_{N_k}}(n) \leq \sum_{j=0}^{l}\binom{n}{j} . \nonumber
\end{equation}
where $l$ is the VC-dimension of $F_{N_k}$. In fact, we can derive $l = \Omega(n)$ when $\alpha$ is a small constant. Assume that $l < n-1$ , then we have $\sum_{j=0}^{l}\binom{n}{j} \leq (e n / l)^{l}$ and $\sum_{i=1}^{\alpha n}\binom{n}{i} \leq (e/\alpha)^{\alpha n}$, so
\begin{equation}
    \frac{2^{n}}{\left(e/\alpha\right)^{\alpha n}} \leq |R| \leq (e n / l)^{l} . \nonumber
\end{equation}
We define a function $h(x)$ as $h(x)=(e/x)^{x}$, then we derive
\begin{equation}
    2 \leq \left(\frac{e}{\alpha}\right)^{\alpha} \left(\frac{e}{l/n}\right)^{l/n} = h(\alpha) h(l/n) . \nonumber
\end{equation}

When $\alpha$ is sufficient small, $l/n \geq C(\alpha)$ that is a constant only depending on $\alpha$, which implies $l=\Omega(n)$. Finally, by using Lemma \ref{lem:vc_dim} and $n = |\mathcal{Q}| = \Omega\left((2\epsilon\sqrt{d/k})^{-k}\right)$, we know $N_k = \Omega\left((2\epsilon\sqrt{d/k})^{-\frac{k}{2}}\right)$. Combined with the definition of balanced distribution, we conclude the proof of Theorem \ref{manifold_lower_bound_appendix}.
\end{proof}

\end{appendix}

%% file: main.bbl
\begin{thebibliography}{66}
\expandafter\ifx\csname natexlab\endcsname\relax\def\natexlab#1{#1}\fi
\expandafter\ifx\csname url\endcsname\relax
  \def\url#1{\texttt{#1}}\fi
\expandafter\ifx\csname urlprefix\endcsname\relax\def\urlprefix{}\fi

\bibitem[{Anthony et~al.(1999)Anthony, Bartlett, Bartlett
  et~al.}]{anthony1999neural}
\text{Anthony, M.}, \text{Bartlett, P.~L.}, \text{Bartlett, P.~L.}
  \text{et~al.} (1999).
\newblock \textit{Neural network learning: Theoretical foundations}, vol.~9.
\newblock cambridge university press Cambridge.

\bibitem[{Baraniuk and Wakin(2009)}]{baraniuk2009random}
\text{Baraniuk, R.~G.} and \text{Wakin, M.~B.} (2009).
\newblock Random projections of smooth manifolds.
\newblock \textit{Foundations of computational mathematics}, \textbf{9} 51--77.

\bibitem[{Bartlett et~al.(2019)Bartlett, Harvey, Liaw and
  Mehrabian}]{bartlett2019nearly}
\text{Bartlett, P.~L.}, \text{Harvey, N.}, \text{Liaw, C.} and \text{Mehrabian,
  A.} (2019).
\newblock Nearly-tight vc-dimension and pseudodimension bounds for piecewise
  linear neural networks.
\newblock \textit{The Journal of Machine Learning Research}, \textbf{20}
  2285--2301.

\bibitem[{Baum(1988)}]{baum1988capabilities}
\text{Baum, E.~B.} (1988).
\newblock On the capabilities of multilayer perceptrons.
\newblock \textit{Journal of complexity}, \textbf{4} 193--215.

\bibitem[{Belkin et~al.(2019)Belkin, Hsu, Ma and
  Mandal}]{belkin2019reconciling}
\text{Belkin, M.}, \text{Hsu, D.}, \text{Ma, S.} and \text{Mandal, S.} (2019).
\newblock Reconciling modern machine-learning practice and the classical
  bias--variance trade-off.
\newblock \textit{Proceedings of the National Academy of Sciences},
  \textbf{116} 15849--15854.

\bibitem[{Bhagoji et~al.(2019)Bhagoji, Cullina and Mittal}]{bhagoji2019lower}
\text{Bhagoji, A.~N.}, \text{Cullina, D.} and \text{Mittal, P.} (2019).
\newblock Lower bounds on adversarial robustness from optimal transport.
\newblock \textit{Advances in Neural Information Processing Systems},
  \textbf{32}.

\bibitem[{Bhattacharjee et~al.(2021)Bhattacharjee, Jha and
  Chaudhuri}]{bhattacharjee2021sample}
\text{Bhattacharjee, R.}, \text{Jha, S.} and \text{Chaudhuri, K.} (2021).
\newblock Sample complexity of robust linear classification on separated data.
\newblock In \textit{International Conference on Machine Learning}. PMLR.

\bibitem[{Biggio et~al.(2013)Biggio, Corona, Maiorca, Nelson, {\v{S}}rndi{\'c},
  Laskov, Giacinto and Roli}]{biggio2013evasion}
\text{Biggio, B.}, \text{Corona, I.}, \text{Maiorca, D.}, \text{Nelson, B.},
  \text{{\v{S}}rndi{\'c}, N.}, \text{Laskov, P.}, \text{Giacinto, G.} and
  \text{Roli, F.} (2013).
\newblock Evasion attacks against machine learning at test time.
\newblock In \textit{Joint European conference on machine learning and
  knowledge discovery in databases}. Springer.

\bibitem[{Bishop(1964)}]{bishop1964relation}
\text{Bishop, R.~L.} (1964).
\newblock A relation between volume, mean curvature and diameter.
\newblock In \textit{Euclidean Quantum Gravity}. World Scientific, 161--161.

\bibitem[{Boser et~al.(1992)Boser, Guyon and Vapnik}]{boser1992training}
\text{Boser, B.~E.}, \text{Guyon, I.~M.} and \text{Vapnik, V.~N.} (1992).
\newblock A training algorithm for optimal margin classifiers.
\newblock In \textit{Proceedings of the fifth annual workshop on Computational
  learning theory}.

\bibitem[{Bubeck et~al.(2020)Bubeck, Eldan, Lee and
  Mikulincer}]{bubeck2020network}
\text{Bubeck, S.}, \text{Eldan, R.}, \text{Lee, Y.~T.} and \text{Mikulincer,
  D.} (2020).
\newblock Network size and weights size for memorization with two-layers neural
  networks.
\newblock \textit{arXiv preprint arXiv:2006.02855}.

\bibitem[{Bubeck et~al.(2021)Bubeck, Li and Nagaraj}]{bubeck2021law}
\text{Bubeck, S.}, \text{Li, Y.} and \text{Nagaraj, D.~M.} (2021).
\newblock A law of robustness for two-layers neural networks.
\newblock In \textit{Conference on Learning Theory}. PMLR.

\bibitem[{Bubeck and Sellke(2021)}]{bubeck2021universal}
\text{Bubeck, S.} and \text{Sellke, M.} (2021).
\newblock A universal law of robustness via isoperimetry.
\newblock \textit{Advances in Neural Information Processing Systems},
  \textbf{34}.

\bibitem[{Chen et~al.(2019)Chen, Jiang, Liao and Zhao}]{chen2019efficient}
\text{Chen, M.}, \text{Jiang, H.}, \text{Liao, W.} and \text{Zhao, T.} (2019).
\newblock Efficient approximation of deep relu networks for functions on low
  dimensional manifolds.
\newblock \textit{Advances in neural information processing systems},
  \textbf{32}.

\bibitem[{Chui and Mhaskar(2018)}]{chui2018deep}
\text{Chui, C.~K.} and \text{Mhaskar, H.~N.} (2018).
\newblock Deep nets for local manifold learning.
\newblock \textit{Frontiers in Applied Mathematics and Statistics}, \textbf{4}
  12.

\bibitem[{Cohen et~al.(2019)Cohen, Rosenfeld and Kolter}]{cohen2019certified}
\text{Cohen, J.}, \text{Rosenfeld, E.} and \text{Kolter, Z.} (2019).
\newblock Certified adversarial robustness via randomized smoothing.
\newblock In \textit{International Conference on Machine Learning}. PMLR.

\bibitem[{Cybenko(1989)}]{cybenko1989approximation}
\text{Cybenko, G.} (1989).
\newblock Approximation by superpositions of a sigmoidal function.
\newblock \textit{Mathematics of control, signals and systems}, \textbf{2}
  303--314.

\bibitem[{Dan et~al.(2020)Dan, Wei and Ravikumar}]{dan2020sharp}
\text{Dan, C.}, \text{Wei, Y.} and \text{Ravikumar, P.} (2020).
\newblock Sharp statistical guaratees for adversarially robust gaussian
  classification.
\newblock In \textit{International Conference on Machine Learning}. PMLR.

\bibitem[{Devlin et~al.(2018)Devlin, Chang, Lee and Toutanova}]{devlin2018bert}
\text{Devlin, J.}, \text{Chang, M.-W.}, \text{Lee, K.} and \text{Toutanova, K.}
  (2018).
\newblock Bert: Pre-training of deep bidirectional transformers for language
  understanding.
\newblock \textit{arXiv preprint arXiv:1810.04805}.

\bibitem[{DeVore et~al.(1989)DeVore, Howard and Micchelli}]{devore1989optimal}
\text{DeVore, R.~A.}, \text{Howard, R.} and \text{Micchelli, C.} (1989).
\newblock Optimal nonlinear approximation.
\newblock \textit{Manuscripta mathematica}, \textbf{63} 469--478.

\bibitem[{Dohmatob(2019)}]{dohmatob2019generalized}
\text{Dohmatob, E.} (2019).
\newblock Generalized no free lunch theorem for adversarial robustness.
\newblock In \textit{International Conference on Machine Learning}. PMLR.

\bibitem[{Gao et~al.(2019)Gao, Cai, Li, Wang, Hsieh and
  Lee}]{gao2019convergence}
\text{Gao, R.}, \text{Cai, T.}, \text{Li, H.}, \text{Wang, L.}, \text{Hsieh,
  C.~J.} and \text{Lee, J.~D.} (2019).
\newblock Convergence of adversarial training in overparametrized neural
  networks.
\newblock \textit{Advances in Neural Information Processing Systems},
  \textbf{32}.

\bibitem[{Gong et~al.(2019)Gong, Boddeti and Jain}]{gong2019intrinsic}
\text{Gong, S.}, \text{Boddeti, V.~N.} and \text{Jain, A.~K.} (2019).
\newblock On the intrinsic dimensionality of image representations.
\newblock In \textit{Proceedings of the IEEE/CVF Conference on Computer Vision
  and Pattern Recognition}.

\bibitem[{Goodfellow et~al.(2014)Goodfellow, Shlens and
  Szegedy}]{goodfellow2014explaining}
\text{Goodfellow, I.~J.}, \text{Shlens, J.} and \text{Szegedy, C.} (2014).
\newblock Explaining and harnessing adversarial examples.
\newblock \textit{arXiv preprint arXiv:1412.6572}.

\bibitem[{Hanin(2019)}]{hanin2019universal}
\text{Hanin, B.} (2019).
\newblock Universal function approximation by deep neural nets with bounded
  width and relu activations.
\newblock \textit{Mathematics}, \textbf{7} 992.

\bibitem[{Hassani and Javanmard(2022)}]{hassani2022curse}
\text{Hassani, H.} and \text{Javanmard, A.} (2022).
\newblock The curse of overparametrization in adversarial training: Precise
  analysis of robust generalization for random features regression.
\newblock \textit{arXiv preprint arXiv:2201.05149}.

\bibitem[{Hinton and van~der Maaten(2008)}]{hinton2008visualizing}
\text{Hinton, G.} and \text{van~der Maaten, L.} (2008).
\newblock Visualizing data using t-sne journal of machine learning research.

\bibitem[{Hornik(1991)}]{hornik1991approximation}
\text{Hornik, K.} (1991).
\newblock Approximation capabilities of multilayer feedforward networks.
\newblock \textit{Neural networks}, \textbf{4} 251--257.

\bibitem[{Jumper et~al.(2021)Jumper, Evans, Pritzel, Green, Figurnov,
  Ronneberger, Tunyasuvunakool, Bates, {\v{Z}}{\'\i}dek, Potapenko
  et~al.}]{jumper2021highly}
\text{Jumper, J.}, \text{Evans, R.}, \text{Pritzel, A.}, \text{Green, T.},
  \text{Figurnov, M.}, \text{Ronneberger, O.}, \text{Tunyasuvunakool, K.},
  \text{Bates, R.}, \text{{\v{Z}}{\'\i}dek, A.}, \text{Potapenko, A.}
  \text{et~al.} (2021).
\newblock Highly accurate protein structure prediction with alphafold.
\newblock \textit{Nature}, \textbf{596} 583--589.

\bibitem[{Karpinski and Macintyre(1995)}]{karpinski1995polynomial}
\text{Karpinski, M.} and \text{Macintyre, A.} (1995).
\newblock Polynomial bounds for vc dimension of sigmoidal neural networks.
\newblock In \textit{Proceedings of the twenty-seventh annual ACM symposium on
  Theory of computing}.

\bibitem[{Liang and Srikant(2017)}]{liang2017deep}
\text{Liang, S.} and \text{Srikant, R.} (2017).
\newblock Why deep neural networks for function approximation?
\newblock In \textit{5th International Conference on Learning Representations,
  ICLR 2017}.

\bibitem[{Liu et~al.(2022)Liu, Chen, Er, Liao, Zhang and
  Zhao}]{liu2022benefits}
\text{Liu, H.}, \text{Chen, M.}, \text{Er, S.}, \text{Liao, W.}, \text{Zhang,
  T.} and \text{Zhao, T.} (2022).
\newblock Benefits of overparameterized convolutional residual networks:
  Function approximation under smoothness constraint.
\newblock \textit{arXiv preprint arXiv:2206.04569}.

\bibitem[{Lu et~al.(2017)Lu, Pu, Wang, Hu and Wang}]{lu2017expressive}
\text{Lu, Z.}, \text{Pu, H.}, \text{Wang, F.}, \text{Hu, Z.} and \text{Wang,
  L.} (2017).
\newblock The expressive power of neural networks: A view from the width.
\newblock \textit{Advances in neural information processing systems},
  \textbf{30}.

\bibitem[{Madry et~al.(2017)Madry, Makelov, Schmidt, Tsipras and
  Vladu}]{madry2017towards}
\text{Madry, A.}, \text{Makelov, A.}, \text{Schmidt, L.}, \text{Tsipras, D.}
  and \text{Vladu, A.} (2017).
\newblock Towards deep learning models resistant to adversarial attacks.
\newblock \textit{arXiv preprint arXiv:1706.06083}.

\bibitem[{Mohri et~al.(2018)Mohri, Rostamizadeh and
  Talwalkar}]{mohri2018foundations}
\text{Mohri, M.}, \text{Rostamizadeh, A.} and \text{Talwalkar, A.} (2018).
\newblock \textit{Foundations of machine learning}.
\newblock MIT press.

\bibitem[{Nakkiran(2019)}]{nakkiran2019adversarial}
\text{Nakkiran, P.} (2019).
\newblock Adversarial robustness may be at odds with simplicity.
\newblock \textit{arXiv preprint arXiv:1901.00532}.

\bibitem[{Neyshabur et~al.(2017)Neyshabur, Bhojanapalli, McAllester and
  Srebro}]{neyshabur2017exploring}
\text{Neyshabur, B.}, \text{Bhojanapalli, S.}, \text{McAllester, D.} and
  \text{Srebro, N.} (2017).
\newblock Exploring generalization in deep learning.
\newblock \textit{Advances in neural information processing systems},
  \textbf{30}.

\bibitem[{Niyogi et~al.(2008)Niyogi, Smale and Weinberger}]{niyogi2008finding}
\text{Niyogi, P.}, \text{Smale, S.} and \text{Weinberger, S.} (2008).
\newblock Finding the homology of submanifolds with high confidence from random
  samples.
\newblock \textit{Discrete \& Computational Geometry}, \textbf{39} 419--441.

\bibitem[{Pearson(1901)}]{pearson1901liii}
\text{Pearson, K.} (1901).
\newblock Liii. on lines and planes of closest fit to systems of points in
  space.
\newblock \textit{The London, Edinburgh, and Dublin philosophical magazine and
  journal of science}, \textbf{2} 559--572.

\bibitem[{Petzka et~al.(2021)Petzka, Kamp, Adilova, Sminchisescu and
  Boley}]{petzka2021relative}
\text{Petzka, H.}, \text{Kamp, M.}, \text{Adilova, L.}, \text{Sminchisescu, C.}
  and \text{Boley, M.} (2021).
\newblock Relative flatness and generalization.
\newblock \textit{Advances in Neural Information Processing Systems},
  \textbf{34} 18420--18432.

\bibitem[{Raghunathan et~al.(2019)Raghunathan, Xie, Yang, Duchi and
  Liang}]{raghunathan2019adversarial}
\text{Raghunathan, A.}, \text{Xie, S.~M.}, \text{Yang, F.}, \text{Duchi, J.~C.}
  and \text{Liang, P.} (2019).
\newblock Adversarial training can hurt generalization.
\newblock \textit{arXiv preprint arXiv:1906.06032}.

\bibitem[{Rajput et~al.(2021)Rajput, Sreenivasan, Papailiopoulos and
  Karbasi}]{rajput2021exponential}
\text{Rajput, S.}, \text{Sreenivasan, K.}, \text{Papailiopoulos, D.} and
  \text{Karbasi, A.} (2021).
\newblock An exponential improvement on the memorization capacity of deep
  threshold networks.
\newblock \textit{Advances in Neural Information Processing Systems},
  \textbf{34}.

\bibitem[{Rice et~al.(2020)Rice, Wong and Kolter}]{rice2020overfitting}
\text{Rice, L.}, \text{Wong, E.} and \text{Kolter, Z.} (2020).
\newblock Overfitting in adversarially robust deep learning.
\newblock In \textit{International Conference on Machine Learning}. PMLR.

\bibitem[{Roweis and Saul(2000)}]{roweis2000nonlinear}
\text{Roweis, S.~T.} and \text{Saul, L.~K.} (2000).
\newblock Nonlinear dimensionality reduction by locally linear embedding.
\newblock \textit{science}, \textbf{290} 2323--2326.

\bibitem[{Schmidt et~al.(2018)Schmidt, Santurkar, Tsipras, Talwar and
  Madry}]{schmidt2018adversarially}
\text{Schmidt, L.}, \text{Santurkar, S.}, \text{Tsipras, D.}, \text{Talwar, K.}
  and \text{Madry, A.} (2018).
\newblock Adversarially robust generalization requires more data.
\newblock \textit{Advances in neural information processing systems},
  \textbf{31}.

\bibitem[{Shafahi et~al.(2018)Shafahi, Huang, Studer, Feizi and
  Goldstein}]{shafahi2018adversarial}
\text{Shafahi, A.}, \text{Huang, W.~R.}, \text{Studer, C.}, \text{Feizi, S.}
  and \text{Goldstein, T.} (2018).
\newblock Are adversarial examples inevitable?
\newblock \textit{arXiv preprint arXiv:1809.02104}.

\bibitem[{Shafahi et~al.(2019)Shafahi, Najibi, Ghiasi, Xu, Dickerson, Studer,
  Davis, Taylor and Goldstein}]{shafahi2019adversarial}
\text{Shafahi, A.}, \text{Najibi, M.}, \text{Ghiasi, M.~A.}, \text{Xu, Z.},
  \text{Dickerson, J.}, \text{Studer, C.}, \text{Davis, L.~S.}, \text{Taylor,
  G.} and \text{Goldstein, T.} (2019).
\newblock Adversarial training for free!
\newblock \textit{Advances in Neural Information Processing Systems},
  \textbf{32}.

\bibitem[{Shaham et~al.(2018)Shaham, Cloninger and
  Coifman}]{shaham2018provable}
\text{Shaham, U.}, \text{Cloninger, A.} and \text{Coifman, R.~R.} (2018).
\newblock Provable approximation properties for deep neural networks.
\newblock \textit{Applied and Computational Harmonic Analysis}, \textbf{44}
  537--557.

\bibitem[{Shen et~al.(2022)Shen, Yang and Zhang}]{shen2022optimal}
\text{Shen, Z.}, \text{Yang, H.} and \text{Zhang, S.} (2022).
\newblock Optimal approximation rate of relu networks in terms of width and
  depth.
\newblock \textit{Journal de Math{\'e}matiques Pures et Appliqu{\'e}es},
  \textbf{157} 101--135.

\bibitem[{Szegedy et~al.(2013)Szegedy, Zaremba, Sutskever, Bruna, Erhan,
  Goodfellow and Fergus}]{szegedy2013intriguing}
\text{Szegedy, C.}, \text{Zaremba, W.}, \text{Sutskever, I.}, \text{Bruna, J.},
  \text{Erhan, D.}, \text{Goodfellow, I.} and \text{Fergus, R.} (2013).
\newblock Intriguing properties of neural networks.
\newblock \textit{arXiv preprint arXiv:1312.6199}.

\bibitem[{Telgarsky(2016)}]{telgarsky2016benefits}
\text{Telgarsky, M.} (2016).
\newblock Benefits of depth in neural networks.
\newblock In \textit{Conference on learning theory}. PMLR.

\bibitem[{Telgarsky(2017)}]{telgarsky2017neural}
\text{Telgarsky, M.} (2017).
\newblock Neural networks and rational functions.
\newblock In \textit{International Conference on Machine Learning}. PMLR.

\bibitem[{Tenenbaum et~al.(2000)Tenenbaum, Silva and
  Langford}]{tenenbaum2000global}
\text{Tenenbaum, J.~B.}, \text{Silva, V.~d.} and \text{Langford, J.~C.} (2000).
\newblock A global geometric framework for nonlinear dimensionality reduction.
\newblock \textit{science}, \textbf{290} 2319--2323.

\bibitem[{Tram{\`e}r et~al.(2018)Tram{\`e}r, Kurakin, Papernot, Goodfellow,
  Boneh and McDaniel}]{tramer2018ensemble}
\text{Tram{\`e}r, F.}, \text{Kurakin, A.}, \text{Papernot, N.},
  \text{Goodfellow, I.}, \text{Boneh, D.} and \text{McDaniel, P.} (2018).
\newblock Ensemble adversarial training: Attacks and defenses.
\newblock In \textit{International Conference on Learning Representations}.

\bibitem[{Tsipras et~al.(2019)Tsipras, Santurkar, Engstrom, Turner and
  Madry}]{tsipras2019robustness}
\text{Tsipras, D.}, \text{Santurkar, S.}, \text{Engstrom, L.}, \text{Turner,
  A.} and \text{Madry, A.} (2019).
\newblock Robustness may be at odds with accuracy.
\newblock In \textit{International Conference on Learning Representations}.
  2019.

\bibitem[{Vardi et~al.(2021)Vardi, Yehudai and Shamir}]{vardi2021optimal}
\text{Vardi, G.}, \text{Yehudai, G.} and \text{Shamir, O.} (2021).
\newblock On the optimal memorization power of relu neural networks.
\newblock \textit{arXiv preprint arXiv:2110.03187}.

\bibitem[{Voulodimos et~al.(2018)Voulodimos, Doulamis, Doulamis and
  Protopapadakis}]{voulodimos2018deep}
\text{Voulodimos, A.}, \text{Doulamis, N.}, \text{Doulamis, A.} and
  \text{Protopapadakis, E.} (2018).
\newblock Deep learning for computer vision: A brief review.
\newblock \textit{Computational intelligence and neuroscience}, \textbf{2018}.

\bibitem[{Wang et~al.(2016)Wang, Yao and Zhao}]{wang2016auto}
\text{Wang, Y.}, \text{Yao, H.} and \text{Zhao, S.} (2016).
\newblock Auto-encoder based dimensionality reduction.
\newblock \textit{Neurocomputing}, \textbf{184} 232--242.

\bibitem[{Yang et~al.(2020)Yang, Rashtchian, Zhang, Salakhutdinov and
  Chaudhuri}]{yang2020closer}
\text{Yang, Y.-Y.}, \text{Rashtchian, C.}, \text{Zhang, H.},
  \text{Salakhutdinov, R.~R.} and \text{Chaudhuri, K.} (2020).
\newblock A closer look at accuracy vs. robustness.
\newblock \textit{Advances in neural information processing systems},
  \textbf{33} 8588--8601.

\bibitem[{Yarotsky(2017)}]{yarotsky2017error}
\text{Yarotsky, D.} (2017).
\newblock Error bounds for approximations with deep relu networks.
\newblock \textit{Neural Networks}, \textbf{94} 103--114.

\bibitem[{Yun et~al.(2019)Yun, Sra and Jadbabaie}]{yun2019small}
\text{Yun, C.}, \text{Sra, S.} and \text{Jadbabaie, A.} (2019).
\newblock Small relu networks are powerful memorizers: a tight analysis of
  memorization capacity.
\newblock \textit{Advances in Neural Information Processing Systems},
  \textbf{32}.

\bibitem[{Zhang et~al.(2021{\natexlab{a}})Zhang, Cai, Lu, He and
  Wang}]{zhang2021towards}
\text{Zhang, B.}, \text{Cai, T.}, \text{Lu, Z.}, \text{He, D.} and \text{Wang,
  L.} (2021{\natexlab{a}}).
\newblock Towards certifying l-infinity robustness using neural networks with
  l-inf-dist neurons.
\newblock In \textit{International Conference on Machine Learning}. PMLR.

\bibitem[{Zhang et~al.(2017)Zhang, Bengio, Hardt, Recht and
  Vinyals}]{zhang2017understanding}
\text{Zhang, C.}, \text{Bengio, S.}, \text{Hardt, M.}, \text{Recht, B.} and
  \text{Vinyals, O.} (2017).
\newblock Understanding deep learning requires rethinking generalization. iclr.
\newblock \textit{arXiv preprint arXiv:1611.03530}.

\bibitem[{Zhang et~al.(2022)Zhang, Wu and Huang}]{zhang2022many}
\text{Zhang, H.}, \text{Wu, Y.} and \text{Huang, H.} (2022).
\newblock How many data are needed for robust learning?
\newblock \textit{arXiv preprint arXiv:2202.11592}.

\bibitem[{Zhang et~al.(2019)Zhang, Yu, Jiao, Xing, El~Ghaoui and
  Jordan}]{zhang2019theoretically}
\text{Zhang, H.}, \text{Yu, Y.}, \text{Jiao, J.}, \text{Xing, E.},
  \text{El~Ghaoui, L.} and \text{Jordan, M.} (2019).
\newblock Theoretically principled trade-off between robustness and accuracy.
\newblock In \textit{International conference on machine learning}. PMLR.

\bibitem[{Zhang et~al.(2021{\natexlab{b}})Zhang, Zhang, Hong, Sun and
  Luo}]{zhang2021expressivity}
\text{Zhang, J.}, \text{Zhang, Y.}, \text{Hong, M.}, \text{Sun, R.} and
  \text{Luo, Z.-Q.} (2021{\natexlab{b}}).
\newblock When expressivity meets trainability: Fewer than $ n $ neurons can
  work.
\newblock \textit{Advances in Neural Information Processing Systems},
  \textbf{34} 9167--9180.

\end{thebibliography}
